\documentclass{article}

\usepackage[english]{babel}
\usepackage[colorlinks,citecolor=blue,urlcolor=blue]{hyperref}

\usepackage[a4paper,top=2cm,bottom=2cm,left=2cm,right=1.7cm,marginparwidth=1.75cm]{geometry}

\usepackage{amsmath}
\usepackage{amssymb}
\usepackage{amsthm}
\usepackage{amsfonts}

\DeclareMathAlphabet\mathbfcal{OMS}{cmsy}{b}{n}

\usepackage{siunitx}
\PassOptionsToPackage{hyphens}{url}\usepackage{hyperref}
\usepackage{cleveref}
\usepackage[utf8]{inputenc}
\usepackage{csquotes}
\usepackage{booktabs}
\usepackage{longtable}
\usepackage{adjustbox}
\usepackage{array}
\usepackage{url}
\usepackage{titlesec}
\usepackage{authblk}
\usepackage{xcolor}

\usepackage[plain,noend]{algorithm2e}

\titleformat{\subsection}
  {\mdseries\itshape\large} 
  {\thesubsection}{1em}{} 

\usepackage[english]{babel}

\DeclareMathOperator{\EX}{\mathbb{E}}

\newcommand{\E}{\mathbb{E}}
\newcommand{\R}{\mathbb{R}}
\newcommand{\bzero}{ {\bf 0} }

\newcommand{\bN}{ {\bf N} }

\newcommand{\bx}{ {\bf x} }

\newcommand{\by}{ {\bf y} }
\newcommand{\bX}{ {\bf X} }
\newcommand{\bR}{ {\bf R} }

\newcommand{\bu}{ {\bf u} }
\newcommand{\bD}{ {\bf D} }

\newcommand{\bW}{ {\bf W} }

\newcommand{\bv}{ {\bf v} }

\newcommand{\br}{ {\bf r} }

\newcommand{\bbeta}{ {\boldsymbol \beta} }
\newcommand{\bmu}{ {\boldsymbol \mu} }

\newcommand{\bOmega}{ {\mathbf \Omega} }

\newcommand{\bzeta}{ {\boldsymbol \zeta} }

\newcommand{\trace}{\mbox{tr}} 

\newcommand{\bpi}{ {\boldsymbol \pi} }

\newcommand{\bRe}{ {\mathbb R} }

\newcommand{\bz}{ {\bf z} }

\newcommand{\argmin}{\mbox{argmin}} 
\newcommand{\argmax}{\mbox{argmax}} 
\newcommand{\sign}{\mbox{sign}} 
 
\newcommand{\diag}{\mbox{diag}}

\usepackage{natbib}
\makeatletter
\newcommand*\bigcdot{\mathpalette\bigcdot@{.5}}
\newcommand*\bigcdot@[2]{\mathbin{\vcenter{\hbox{\scalebox{#2}{$\m@th#1\bullet$}}}}}
\makeatother

\makeatletter
\newcommand{\vast}{\bBigg@{3}}
\newcommand{\Vast}{\bBigg@{4}}
\makeatother

\newcommand\hbound{\underline{h}}
\newcommand\ellbound{\underline{\ell}}

\DeclareMathOperator{\sech}{sech}
\DeclareRobustCommand{\bbone}{\text{\usefont{U}{bbold}{m}{n}1}}

\def\simind{\stackrel{\mbox{\scriptsize{\rm ind}}}{\sim}}

\newtheorem{Theorem}{\bf Theorem}
\newtheorem{Lemma}[Theorem]{\bf Lemma}
\newtheorem{Proposition}[Theorem]{\bf Proposition}
\newtheorem{Corollary}[Theorem]{\bf Corollary}

\crefformat{figure}{#2Figure~#1#3}
\Crefformat{figure}{#2Figure~#1#3}
\crefformat{table}{#2Table~#1#3}
\Crefformat{table}{#2Table~#1#3}
\crefformat{section}{#2Section~#1#3}
\Crefformat{section}{#2Section~#1#3}

\author[1]{{\Large Niccol\'o Anceschi}}
\author[2]{{\Large Cristian Castiglione}}
\author[3]{{\Large Tommaso Rigon}}
\author[4]{{\Large Giacomo Zanella}}
\author[4]{{\Large Daniele Durante}}

\affil[1]{Department of Statistical Science, Duke University, Durham,
North Carolina 27708, U.S.A.
niccolo.anceschi@duke.edu (corresponding author)}
\affil[2]{Institute for Data Science and Analytics, Bocconi University, Via Röntgen 1, 20136 Milan, IT. cristian.castiglione@unibocconi.it}
\affil[3]{Department of Economics, Management and Statistics, University of Milano–Bicocca,
Piazza dell’Ateneo Nuovo 1, 20126 Milano, IT. tommaso.rigon@unimib.it}
\affil[4]{Department of Decision Sciences, Bocconi University, Via Röntgen 1, 20136 Milan, IT. giacomo.zanella@unibocconi.it, daniele.durante@unibocconi.it}

\title{\huge Optimal and computationally tractable lower bounds \\ for logistic log-likelihoods}

\date{}

\begin{document}
\maketitle

\begin{abstract}
\fontsize{11}{12}\selectfont
The logit transform is arguably the most widely-employed link function beyond linear~settings. This transformation routinely appears in regression models for binary data and provides a central building-block in popular methods for both classification and regression.~Its widespread~use,~combined with the lack of analytical solutions for the optimization of  objective functions involving~the logit transform, still motivates active research in computational statistics. Among the directions explored, a central one has focused on the design of tangent lower bounds for logistic log-likelihoods that can be tractably optimized, while providing a tight approximation of these log-likelihoods. This has led~to the development of effective minorize-maximize (\textsc{mm}) algorithms~for point estimation, and variational  schemes for approximate~Bayesian inference~under~several~logit models. However, the overarching focus has been on tangent quadratic minorizers. In fact, it is still unclear whether tangent lower bounds sharper than quadratic ones can be derived without undermining the tractability of the resulting minorizer.  This article~addresses~such~a~question through the  design and study of a novel piece-wise quadratic lower bound that uniformly improves any tangent quadratic minorizer, including the sharpest ones, while~admitting~a~direct interpretation in terms of the classical generalized lasso problem. As illustrated in realistic empirical studies, such a sharper bound not only improves the speed of convergence of common \textsc{mm} schemes for penalized maximum likelihood estimation, but also yields tractable variational Bayes (\textsc{vb}) approximations with higher accuracy relative to those obtained under popular quadratic bounds employed in \textsc{vb}.
\vspace{9pt}
\\
\textbf{Keywords:} Logit link, Minorize-Maximize; Piece-wise Quadratic Bounds; Tangent Minorizer; Variational Bayes
\end{abstract}

\fontsize{12}{15}\selectfont

\vspace{5pt}

\section{Introduction}\label{sec_intro_1}
Statistical models lacking analytical results for inference on the corresponding parameters~are~common in both frequentist and Bayesian settings. For example, even basic logistic regression~requires iterative procedures for point estimation. In this context, a convenient strategy to overcome these challenges is to iteratively approximate the log-likelihood of such models through accurate tangent lower bounds admitting tractable maximization. Such a perspective has led to effective variational Bayes (\textsc{vb}) approximations optimizing tractable  lower~bounds of the marginal~likelihood \citep[e.g.,][]{Blei_2017}, and popular expectation-maximization  (\textsc{em}) \citep{em_McLachlan1996} and minorize-maximize (\textsc{mm}) \citep{mm_Hunter2004} schemes for maximum likelihood estimation. In both cases, the resulting procedures face a fundamental  trade-off, determined by the selected class~of lower bounds. Larger and sharper classes are expected to accurately approximate the target log-likelihood at the expense of computational challenges in optimizing~the~selected bound. Conversely, simpler~classes  mitigate these tractability issues, but the resulting~approximation suffers~from reduced accuracy.  Depending on the inference goal,~this~trade-off has different consequences. For example, in \textsc{vb} it   affects the quality of the approximation of the target posterior, while in \textsc{mm} and \textsc{em} schemes it controls the efficiency of the optimization routine, which depends both on the number of iterations and  the cost of each update \citep[][]{mm_Wu2010}.

In this article, we focus on studying and improving several classes of tangent lower  bounds for a family of log-likelihoods that covers a central role in statistics, namely logistic log-likelihoods. In this framework, two popular lower bounds have been developed by \citet{Bohning1988} and \citet{Jaakkola_Jordan_2000} with a focus on tractable quadratic minorizers. The former (\textsc{bl}) exploits a uniform bound on the curvature of the logistic log-likelihood to minorize~its~Hessian, thereby obtaining a tractable lower bound that has been  widely implemented in the literature~\citep[see, e.g.,][]{mm_Hunter2004,mm_Wu2010,khan12_multinomial}, and subsequently extended to multinomial logit \citep[e.g.,][]{Bohning1992MultinomialLR,  Multivariate_sharp_Qbounds, Knowles_Minka_NIPS2011}. The latter (\textsc{pg}) leverages instead a supporting hyperplane inequality~to~obtain~a bound that is still quadratic, yet provably~sharper than the \textsc{bl}  one \citep{DeLeeuw_09_sharQ}. This~tighter~approximation has motivated a widespread use within the \textsc{em},  \textsc{mm} and \textsc{vb} literature \citep[see, e.g.,][]{Lee2010SPARSELP, Ren2011LogisticSP, Carbonetto2012ScalableVI}, along with studies \citep{Durante_Rigon_2019} proving~a~direct link among the \textsc{pg} bound and the P\'{o}lya-Gamma  data augmentation \citep{Polson2012}.

While the impact of \textsc{bl} and \textsc{pg} minorizers should have motivated further improvements of~such bounds, research along this direction has been~limited. In fact, as clarified in Section~\ref{subsec:quadratic_bounds},  \textsc{pg}~is~optimal among~the~tangent quadratic minorizers of logistic log-likelihoods. Thus, sharper solutions in the quadratic family cannot~be derived. Conversely, it is still unclear how this class can~be~enlarged for obtaining sharper, yet tractable, alternatives. In Section~\ref{sec_PLQ} we address this~gap~by~designing~a provably-sharper piece-wise quadratic (\textsc{pq}) minorizer that is optimal in a class of tangent lower bounds broader than the quadratic one,  while preserving tractability. This bound arises from a improvement over the classical supporting hyperplane inequality, which leads to  a tangent minorizer that can~be~derived analytically and gains sharpness via $L_1$ terms. Such a tractability~is~in~contrast with available piece-wise quadratic minorizers \citep[][]{Marlin_ICML2011_Piecewise_Q_bounds,khan12_multinomial}, which cannot be obtained analytically. In addition, as clarified in Section~\ref{sec_PLQ}, the \textsc{pq} minorizer we derive  can be interpreted as  a standard  generalized lasso problem \citep{gen_lasso_solution_path_2011}. This  facilitates optimization of the proposed minorizer through available schemes for generalized~lasso. 
The broad scope and practical advantages of this bound are illustrated in Section~\ref{sec_Penalized_Log_Reg} with a focus~on penalized maximum likelihood estimation and \textsc{vb}. In the former,~the proposed method substantially reduces the iterations to convergence and the total execution time, while in the latter~it~yields a more accurate approximation of the target posterior without increasing the computational costs. This latter result motivates extensive use of the proposed \textsc{pq} bound in variational inference as an improved alternative to \textsc{pg}. Proofs and further results can be found in the Supplementary Material.

\vspace{5pt}

\section{Tangent lower bounds of logistic log-likelihoods}
\label{sec_logistic_bounds}
\vspace{-3pt}
Let $\by = (y_1,\dots,y_n)^\top$ denote a vector of  independent  Bernoulli variables with success probabilities $\pi_i = \pi(\bx_i^\top \bbeta) = (1+e^{-\bx_i^\top \bbeta})^{-1}$, depending on a vector $\bbeta = (\beta_1,\dots,\beta_p)^\top \in \bRe^p$ of regression parameters  and on a set of observed predictors $\bx_i = (x_{i1},\dots,x_{ip})^\top \in \bRe^p $, for $i=1,\dots,n$.~The focus of this article is on finding tight but tractable tangent lower bounds of the induced log-likelihood  
\begin{align}\label{log_lik_log}
\ell(\bbeta) &= \sum\nolimits_{i=1}^n \log p(y_i \mid \bbeta, \bx_i) = \sum\nolimits_{i=1}^n \big[y_i \, \bx_i^\top \bbeta - \log ( 1+e^{\bx_i^\top \bbeta} ) \big].
\end{align}
Rewriting each  term $\log p(y_i \mid \bbeta, \bx_i)$ as $(y_i -0.5) \bx_i^\top \bbeta + h(\bx_i^\top \bbeta)$ with $h({r}) = -\log(e^{r/2}+e^{-r/2})$ for $r\in\bRe$, the construction of a global tangent minorizer for $\ell(\bbeta)$ can proceed by lower bounding the one-dimensional function $h$. In particular, we will focus our attention on constructing tangent lower bounds $\hbound:(r, \zeta)\in\bRe^2\mapsto \hbound(r \mid \zeta)\in\bRe$ such that
\begin{equation}\label{mm_def_minorization}
\begin{split}
    h(\zeta) &= \hbound(\zeta \mid \zeta)\qquad \hbox{and}\qquad
    h(r) \geq \hbound(r \mid \zeta) \qquad \forall \, r,\zeta \in \bRe \, .\vspace{-1pt}
\end{split}
\end{equation}
Equivalently, $\hbound$ defines a family of lower bounds indexed by a parameter $\zeta$, which denotes the location where $\hbound(\cdot\mid\zeta)$ is tangent to $h$.
Setting the tangent location for the $i$-th term  to \smash{$\zeta_i=\bx_i^\top \widetilde{\bbeta}$}, 
leads to a bound for $\ell(\bbeta)$ of the form
\begin{align}\label{eq:glob_bound}
\ell(\bbeta) &\geq 
\ellbound(\bbeta \mid \widetilde{\bbeta})=
\sum\nolimits_{i=1}^n \big[ (y_i -0.5) \bx_i^\top \bbeta + \hbound(\bx_i^\top \bbeta\mid  \bx_i^\top \widetilde{\bbeta}) \big]
\,,& \forall \, \bbeta\in\bRe^p,  \ \widetilde{\bbeta}\in\bRe^p\,.
\end{align}
The value of $\widetilde{\bbeta}$ or, equivalently, of the tangent locations $\{ \zeta_i \}_{i=1}^n$, can then be updated iteratively in order to optimize the bound in \eqref{eq:glob_bound} according to some criterion of interest, as done in, e.g., \textsc{mm} and \textsc{vb} schemes. Clearly, to be useful in practice, $\hbound$ (and hence $\ellbound$) must be available analytically at each update and should be designed in a way that allows for tractable maximization. As discussed in Section~\ref{subsec:quadratic_bounds}, this has motivated a major focus in the literature on tangent~quadratic~lower~bounds.

\vspace{6pt}

\section{Tangent quadratic lower bounds}
\label{subsec:quadratic_bounds}
\vspace{-1pt}
Within the class of tangent quadratic minorizers of logistic log-likelihoods, the one proposed by \citet{Bohning1988} (\textsc{bl})  provides a seminal construction that exploits a uniform bound to the curvature of the logistic log-likelihood. Focusing on the one-dimensional function~$h$,~this~yields
\begin{equation}\label{bl_bound_h}
    h(r) 
    \geq
    \hbound_{\textsc{bl}}(r \mid \zeta) 
    := 
    h(\zeta) + 
    h'(\zeta) (r-\zeta) - 0.25(r-\zeta)^2/2 \, ,
\end{equation}
which follows by noticing that $h''(r)\in [ -0.25,0 )$, $\forall \; r\in\bRe$. Albeit providing a tractable~minorizer, such a bound can be further improved in terms of quality in approximating $h$ within~the tangent quadratic class. Such a direction has been explored in \citet{Jaakkola_Jordan_2000}~by~reparametrizing $h(r)$ as a function of the squared linear predictor $\rho = r^2$. This yields a function $\tilde{h}(\rho) := h(\sqrt{\rho})$ which is convex in $\rho$ and, hence, can be lower bounded with its tangent line at any given location. Therefore
\begin{equation}\label{eq_JJ_transformed}
    \tilde{h}(\rho) 
    \geq
    \tilde{h}(\varphi) + 
    \tilde{h}'(\varphi) (\rho-\varphi),
\end{equation}
for any location $\varphi \in \mathbb{R}^+$. The same relation holds true also when transforming  the problem into the original space, leading to the \textsc{pg} bound defined as
\begin{equation}\label{eq_JJ_transformed_2}
\begin{split}
    h(r) \geq \hbound_{\textsc{pg}}(r \mid \zeta) &:= \tilde{h}(\zeta^2) + 
   \tilde{h}' (\zeta^2) (r^2-\zeta^2) = h(\zeta) - \tanh(\zeta/2) (r^2-\zeta^2)/(4 \zeta) \\
    & \; = h(\zeta) +  
    h'(\zeta) (r-\zeta) - w_{\textsc{pg}}(\zeta)(r-\zeta)^2/2\,,\end{split}
\end{equation}
where $w_{\textsc{pg}}(\zeta) =  (2 \zeta)^{-1}\tanh(\zeta/2)$. Since $w_{\textsc{pg}}(\zeta)\in (0, 0.25] $ for any $\zeta \in \bRe$, the \textsc{bl} bound~is~a~tangent minorizer of the \textsc{pg} one. Hence, $\hbound_{\textsc{pg}}$ uniformly dominates $\hbound_{\textsc{bl}}$, thereby providing~a~more~accurate characterization of the target $\ell$, while preserving the tractability of the quadratic bounds. Although being developed under purely mathematical arguments, as shown by \citet{Durante_Rigon_2019},  such a bound admits a direct interpretation under the P\'{o}lya-Gamma data-augmentation \citep{Polson2012}. This facilitates the development of \textsc{em} algorithms for maximum likelihood estimation under $\ell(\bbeta)$ and closed-form coordinate ascent variational inference schemes for approximate Bayesian inference on $\bbeta$, under Gaussian prior (see the Supplementary Material~for~more details).

\vspace{5pt}
\subsection{Optimality properties of the quadratic \textsc{pg} lower bound}\label{PG_optimality}
When designing a tangent quadratic minorizer $\hbound$ for $h$ satisfying \eqref{mm_def_minorization}, the only tunable quantity is  $\hbound''(r \mid \zeta)$. Hence,  the relative tightness~of~two quadratic minorizers only depends~on~the~respective curvatures. This allowed us to show that $\hbound_{\textsc{pg}}(r \mid \zeta) \geq \hbound_{\textsc{bl}}(r \mid \zeta)$ for any $r,\zeta \in \bRe$.~Lemma~\ref{Lemma_optimality_pg_1D_space}~extends such result by stating \textsc{pg}  optimality in the family of all  tangent quadratic minorizers, beyond~\textsc{bl}.
\begin{Lemma}\label{Lemma_optimality_pg_1D_space}
\itshape Define the  family of all quadratic tangent minorizers for $h$ as
\begin{equation*}
    \mathcal{H}_{\textsc{q}} = \vast\{ \;
    \begin{aligned}
        & \hspace{40pt} \hbound: \bRe\times\bRe\rightarrow\bRe \quad \text{s.t.} \; \eqref{mm_def_minorization} \; \hbox{holds and}
        \\
        &\hbound(r \mid \zeta)= a(\zeta) + b(\zeta)r +c(\zeta)r^2 \quad \text{for some} \quad a,b,c
    \end{aligned}
    \; \vast\}\,.
\end{equation*}
Then, for any $\hbound_{\textsc{q}}\in \mathcal{H}_{\textsc{q}}$, it holds that    $ h(r) \geq \hbound_{\textsc{pg}}(r \mid \zeta) \geq \hbound_{\textsc{q}}(r \mid \zeta)$, $\forall \, r,\zeta \in \bRe$.
\end{Lemma}
The above result is directly related to the sharpness property proved by \citet{DeLeeuw_09_sharQ} for $\hbound_{\textsc{pg}}$ leveraging the symmetry of the target function and of $\hbound_{\textsc{pg}}(r \mid \zeta)$. In the Supplementary Material we provide a novel proof which gives a more direct intuition and set the foundations to develop the sharper piece-wise minorizer proposed in Section~\ref{sec_PLQ}. Notice that Lemma~\ref{Lemma_optimality_pg_1D_space} directly translates~into~an~optimality result for the \textsc{pg} bound among the quadratic and separable tangent minorizers of the target logistic log-likelihood $\ell$ in the  $\bbeta$ space. In particular, replacing $h$ with $\hbound_{\textsc{pg}}$ in \eqref{eq:glob_bound} yields a tangent quadratic minorizer $\ellbound_{\textsc{pg}}$ of $\ell$ with the optimality properties  in Corollary~\ref{Corollary_optimality_pg_beta_space}.

  \begin{Corollary}\label{Corollary_optimality_pg_beta_space}
\itshape Define the  family of all quadratic and separable tangent minorizers for $\ell$ as
\begin{equation*}
    \mathcal{G}_{\textsc{q}} \hspace{-3pt} \, = \, \hspace{-3pt} \vast\{ \;
    \begin{aligned}
        & \ellbound: \bRe^p\times\bRe^p \rightarrow \bRe \;\; \hbox{s.t.} \;\; \ell(\widetilde{\bbeta}) = \ellbound(\widetilde{\bbeta} \mid \widetilde{\bbeta}) \;\; \text{and} \;\;
        \ell(\bbeta) \geq \ellbound(\bbeta \mid \widetilde{\bbeta}) \;\; \forall \bbeta,\widetilde{\bbeta} \in \bRe^p, \; \text{and}
        \\
        &
        \ellbound(\bbeta \mid \widetilde{\bbeta}) = \textstyle\sum\nolimits_{i=1}^n [ (y_i -0.5) \bx_i^\top \bbeta + \hbound_{\textsc{q},i}(\bx_i^\top \bbeta\mid  \bx_i^\top \widetilde{\bbeta}) ], \ \text{with} \ \hbound_{\textsc{q},i}  \in \mathcal{H}_{\textsc{q}}, \forall \; i=1, \ldots, n
    \end{aligned}
    \, \vast\} \hspace{-1pt}.
\end{equation*}
Then, for any $\ellbound_{\textsc{q}}\in \mathcal{G}_{\textsc{q}}$, it holds that    $\ell(\bbeta) \geq \ellbound_{\textsc{pg}}(\bbeta \mid \widetilde{\bbeta}) \geq \ellbound_{\textsc{q}}(\bbeta \mid \widetilde{\bbeta})$, $\forall \, \bbeta,\widetilde{\bbeta} \in \bRe^p$.
\end{Corollary}

Section~\ref{sec_PLQ} develops a novel minorizer that is even sharper than \textsc{pg} and preserves tractability.

\vspace{7pt}
\section{A novel piece-wise quadratic (\textsc{PQ}) lower bound} \label{sec_PLQ}
While tractable, tangent quadratic minorizers may have limited accuracy due~to the inability to control quantities beyond the curvature terms. This has motivated a focus on more sophisticated tangent minorizers, including piece-wise quadratic ones \citep[see, e.g.,][]{Marlin_ICML2011_Piecewise_Q_bounds,khan12_multinomial}.
However, as discussed in Section~\ref{sec_intro_1}, advances in this direction have been limited due to the lack of simple and tractable solutions. In this section we introduce a novel piece-wise quadratic (\textsc{pq}) tangent lower bound of the logistic log-likelihood, that is provably~sharper~than~\textsc{pg}~and~improves the tractability of the available piece-wise quadratic minorizers. In particular, unlike the bounds in \citet{Marlin_ICML2011_Piecewise_Q_bounds} and \citet{khan12_multinomial} it does not require solving an internal~optimization problem to derive  the bound itself, but rather is directly available as for \textsc{bl}~and~\textsc{pq}.

Assuming again a separable structure as in \eqref{eq:glob_bound} for the overall bound $\ellbound$, we derive the proposed \textsc{pq} tangent minorizer by  complementing each quadratic term with an additional piece-wise linear contribution, proportional to the $L_1$-norm $\lvert \bx_i^\top \bbeta \rvert$ of the linear predictor $ \bx_i^\top \bbeta \in \mathbb{R}$. More specifically, working under the same transformed space as in \citet{Jaakkola_Jordan_2000}, it can be noticed~that the tightness of the \textsc{pg} minorization can be possibly improved by including non-linear terms~in~the right-hand-side of  \eqref{eq_JJ_transformed}. As clarified in the following, including a term proportional to $\sqrt{\rho}$ achieves an effective balance between increased sharpness and limited reduction in tractability. This yields
\begin{equation*}
    \tilde{\hbound}_{\textsc{pq}}(\rho \mid \varphi) 
    := 
    \tilde{h}(\varphi) -  \widetilde{w}_{\textsc{pq}}(\varphi) (\rho-\varphi)/2 - \widetilde{\nu}_{\textsc{pq}}(\varphi)(\sqrt{\rho}- \sqrt{\varphi}).
\end{equation*} 
While \smash{$\tilde{h}(\varphi) = \tilde{\hbound}_{\textsc{pq}}(\varphi \mid \varphi) $} is clearly satisfied, \smash{$\tilde{h}(\rho) \geq 
\tilde{\hbound}_{\textsc{pq}}(\rho \mid \varphi) $} must be ensured by constraining the coefficients $ \widetilde{w}_{\textsc{pq}}(\varphi) $ and $\widetilde{\nu}_{\textsc{pq}}(\varphi)$. To this end, recall that the continuity of $\tilde{h}(\rho) $ and $\tilde{\hbound}_{\textsc{pq}}(\rho \mid \varphi)$, together with \smash{$\tilde{h}(\varphi) = \tilde{\hbound}_{\textsc{pq}}(\varphi \mid \varphi) $}, imposes the first constraint \smash{$[\partial\tilde{\hbound}_{\textsc{pq}}(\rho \mid \varphi)/\partial \rho] | _{\rho=\varphi}=\tilde{h}'(\varphi)$}. As for the second, note that $\widetilde{\nu}_{\textsc{pq}}(\varphi)=0$ restores the $\textsc{pg}$ bound, while $\widetilde{\nu}_{\textsc{pq}}(\varphi) < 0 $ leads to a worse concave minorizer bounding $\tilde{\hbound}_{\textsc{pg}}$ from below.  Conversely, any $\widetilde{\nu}_{\textsc{pq}}(\varphi)>0 $ yields a convex function, having \smash{$\tilde{\hbound}_{\textsc{pg}}$} as a tangent lower bound (see Figure~\ref{figure_plq_r2}). However, excessively~large values of $\widetilde{\nu}_{\textsc{pq}}(\varphi)$ would violate \smash{$\tilde{h}(\rho) \geq \tilde{\hbound}_{\textsc{pq}}(\rho \mid \varphi) $}. A solution is to progressively increase $\widetilde{\nu}_{\textsc{pq}}(\varphi)>0$~until~the~first~point~of contact between $\tilde{h}(\rho)$ and \smash{$\tilde{\hbound}_{\textsc{pq}}(\rho \mid \varphi)$}. As clarified in Proposition~\ref{prop:PLQ_dom_PG} and Lemma~\ref{Lemma_optimality_plq}, setting the point of contact at $0$ (i.e., including the constraint \smash{$\tilde{h}(0) = \tilde{\hbound}_{\textsc{pq}}(0 \mid \varphi)$}) yields the desired tangent piece-wise quadratic minorizer with optimality properties. 

Combining the two above constraints, we obtain $\widetilde{w}_{\textsc{pq}}(\varphi)= 2(\tilde{h}(\varphi) - \tilde{h}(0) - 2 \, \varphi \, \tilde{h}'(\varphi) )/\varphi$ and \smash{$   \widetilde{\nu}_{\textsc{pq}}(\varphi)= - 2( \tilde{h}(\varphi) - \tilde{h}(0) - \varphi \, \tilde{h}'(\varphi) )/\sqrt{\varphi}$}. In the original parametrization this yields
\begin{align}
\label{plq_original_space}
    &  \hbound_{\textsc{pq}}(r \mid \zeta) 
   := h(\zeta) -  w_{\textsc{pq}}(\zeta) (r^2-\zeta^2)/2 - \nu_{\textsc{pq}}(\zeta)(\lvert r \rvert- \lvert \zeta \rvert),\\
    w_{\textsc{pq}}(\zeta)&= 2 \, w_{\textsc{pg}}(\zeta) - 2 \log \cosh \big(\zeta/2 \big) /\zeta^2\,,\qquad
    \nu_{\textsc{pq}}(\zeta)= \lvert \zeta \rvert \big(w_{\textsc{pg}}(\zeta)-w_{\textsc{pq}}(\zeta) \big) \, .   \nonumber 
\end{align}
As suggested by the above discussion, the \textsc{pq} solution is expected to minorize $h(r)$ through~a bound that uniformly dominates the \textsc{pg} one.  These results are formalized in Proposition~\ref{prop:PLQ_dom_PG}.~As~clarified in the proof in the Supplementary Material, the fact that $\hbound_{\textsc{pq}}$ is a lower bound to $h$ relies~on the symmetry of $h$ and, more importantly, on the fact that its curvature is monotone~increasing with $r$.  In this sense, the proof of Proposition~\ref{prop:PLQ_dom_PG} requires more subtle arguments relative~to~classical inequalities employed in the derivation of standard tangent lower bounds.

\begin{figure}[t!]
\centering
\raisebox{0.1cm}{\includegraphics[width=0.48\textwidth]{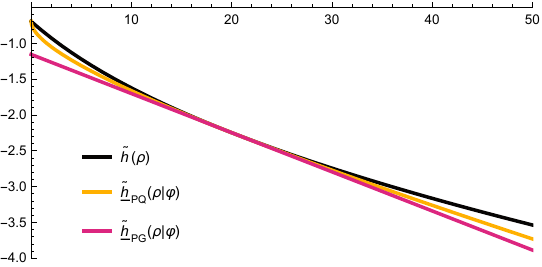}}
\hspace{2pt}
\includegraphics[width=0.48\textwidth]{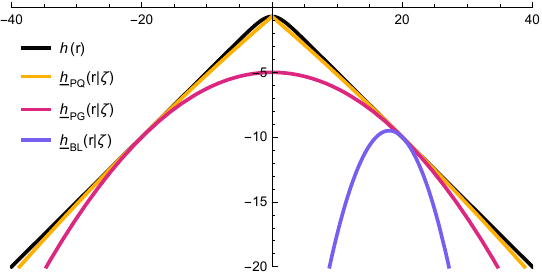}
  \caption{Left: comparison between  $\tilde{\hbound}_{\textsc{pq}}(\rho \mid \varphi)$,  $\tilde{\hbound}_{\textsc{pg}}(\rho \mid \varphi)$ and $\tilde{h}(\rho)$ as a function of $\rho = r^2$, with $\varphi=20$.  Right:~Comparison between   $\hbound_{\textsc{pq}}(r \mid \zeta)$,      $\hbound_{\textsc{pg}}(r \mid \zeta)$, $\hbound_{\textsc{bl}}(r \mid \zeta)$ and $h(r)$ as a function of $r$, with $\zeta=20$.}
  \vspace{-5pt}
    \label{figure_plq_r2}
\end{figure}

\begin{Proposition}\label{prop:PLQ_dom_PG}
\itshape Define $\hbound_{\textsc{pq}}(r \mid \zeta) $ as in \eqref{plq_original_space}. Then,  $h(r) \geq \hbound_{\textsc{pq}}(r \mid \zeta) \geq \hbound_{\textsc{pg}}(r \mid \zeta)$, $\forall \, r,\zeta\in\bRe$.
\end{Proposition}

More generally, it is possibile to state the optimality property for \textsc{pq} in Lemma~\ref{Lemma_optimality_plq}. 
\begin{Lemma}\label{Lemma_optimality_plq}
\itshape Define the  family of all $L_1$-augmented tangent quadratic minorizers for $h$ as
\begin{equation*}
    \mathcal{H}_{\textsc{s}} = \vast\{ \;
    \begin{aligned}
        & \hspace{65pt} \hbound: \bRe\times\bRe\rightarrow\bRe \quad \text{s.t.} \; \eqref{mm_def_minorization} \; \hbox{holds and}
        \\
        &\hbound(r \mid \zeta)= a(\zeta) + b(\zeta)r +c(\zeta)r^2+d(\zeta)|r| \quad \text{for some} \quad a,b,c,d 
    \end{aligned}
    \; \vast\}\,.
\end{equation*}
Then, for any $\hbound_{\textsc{s}}\in \mathcal{H}_{\textsc{s}}$, it holds that $h(r) \geq  \hbound_{\textsc{pq}}(r \mid \zeta) \geq \hbound_{\textsc{s}}(r \mid \zeta)$, $\forall \, r, \zeta \in \bRe$.
\end{Lemma}
The accuracy gain of \textsc{pq} relative to \textsc{pg} and $\textsc{bl}$ is illustrated in Figure \ref{figure_plq_r2}, for $\zeta =20$. The~three bounds coincide in the limit of $\vert \zeta \vert$ going to $0$. Conversely, the larger $\vert \zeta \vert$, the more remarkable is the relative improvement in the approximation accuracy given by the \textsc{pq} bound. Notice that, similarly to Corollary~\ref{Corollary_optimality_pg_beta_space}, replacing $h$ with $\hbound_{\textsc{pq}}$ in \eqref{eq:glob_bound} yields a tangent  minorizer $\ellbound_{\textsc{pq}}$ of $\ell$ which directly inherits the optimality properties of $\hbound_{\textsc{pq}}$ stated in Lemma~\ref{Lemma_optimality_plq}. As clarified~in~Section~\ref{sec_plq_gen_lasso},~this~minorizer~also admits a direct interpretation as the negative loss of a generalized~lasso~problem.

\subsection{Interpretation and connection with generalized lasso}\label{sec_plq_gen_lasso}

Employing  $\hbound_{\textsc{pq}}(r\mid\zeta)$ in \eqref{eq:glob_bound} results in an interpretable tangent minorizer \smash{$\ellbound_{\textsc{pq}}(\bbeta \mid \widetilde{\bbeta})$} for~the~logistic log-likelihood. In particular, collecting in the constant $c$ all terms not depending on $\bbeta$,~yields
\begin{equation}\label{eq_plq_bound_beta_intercept}
\begin{split}
\ell(\bbeta)\geq \ellbound_{\textsc{pq}}(\bbeta \mid \widetilde{\bbeta}) 
&=- 0.5\cdot(\by_{\textsc{pq}} -\bX \bbeta)^\top\bW_{\textsc{pq}}(\by_{\textsc{pq}} -\bX \bbeta)-\lVert \bN_{\textsc{pq}} \bX \bbeta \rVert_1 + c,
\end{split}
\end{equation}
where \smash{$\bW_{\textsc{pq}} = \text{diag}\big(\{ w_{\textsc{pq}}( \bx_i^\top \widetilde{\bbeta} ) \}_{i=1}^n \big)$}, $\bN_{\textsc{pq}}= \text{diag}\big(\{ \nu_{\textsc{pq}}( \bx_i^\top \widetilde{\bbeta} ) \}_{i=1}^n \big)$, while ${\by}_{\textsc{pq}} =\bW^{-1}_{\textsc{pq}} (\by-0.5 \cdot {\bf 1}_n)$. Hence, the  loss \smash{$- \ellbound_{\textsc{pq}}(\bbeta \mid \widetilde{\bbeta})$} coincides with that of weighted least squares under the generalized lasso penalty $\lVert \bD \bbeta \rVert_1$, where \smash{$\bD = \bN_{\textsc{pq}} \bX$}. Recalling  \citet{gen_lasso_solution_path_2011}, such a penalty introduces a regularization on  linear combinations $ \bD \bbeta$, rather than on $\bbeta$. In~our case, $\bD = \bN_{\textsc{pq}} \bX$ which essentially enforces a penalization on those $\bbeta$ yielding large values of $\bX \bbeta$. This constraint~is strengthened by the monotonicity of the multiplicative terms \smash{$\nu_{\textsc{pq}}(\bx_i^\top \widetilde{\bbeta} )$} with respect to \smash{$  \vert \bx_i^\top \widetilde{\bbeta} \vert$}.  

The above connection allows to inherit directly any result on generalized lasso~\citep[][]{gen_lasso_solution_path_2011,gen_lasso_dual_path_2016} for optimizing \smash{$ \ellbound_{\textsc{pq}}(\bbeta \mid \widetilde{\bbeta}) $} within, e.g.,~\textsc{mm} and \textsc{vb}. 

\vspace{5pt}

\section{Applicability of the novel PQ bound and empirical assessments}
\vspace{-1pt}
\label{sec_Penalized_Log_Reg}
Recalling Section~\ref{sec_intro_1},  tangent minorizers have broad potential in both frequentist and Bayesian statistics. Here we showcase the applicability of the three bounds under analysis and the empirical improvements achieved by the proposed \textsc{pq} minorizer in the context of penalized maximum~likelihood estimation and variational Bayes approximation, with focus on spatial logistic regression. 

\subsection{Penalized maximum likelihood estimation}
\label{pen_li}
\vspace{-2pt}
Penalized maximum likelihood estimation solves $\argmax_{\bbeta}[\ell(\bbeta) - P_\lambda(\bbeta)]$ under a preselected penalty $P_\lambda(\bbeta)$. For the sake of generality,  we consider the generalized elastic-net penalty 
$P_\lambda(\bbeta) = \lambda[\alpha \lVert \bD\bbeta \rVert_1 {+} 0.5{\cdot}(1-\alpha)\lVert \bD\bbeta \rVert_2^2]$, 
where $\lambda>0$ is a shrinkage parameter, $\alpha\in[0{,}1]$ balances the $L_1$ and $L_2$ terms, while $\bD$ defines the linear combinations of $\bbeta$ that are subject to penalization; see e.g., \citet{helwig2025versatile} for an example focused on group elastic-net. In solving $\argmax_{\bbeta}[\ell(\bbeta) - P_\lambda(\bbeta)]$, notice that if $\ell(\bbeta)$ were quadratic such an optimization would reduce to  a standard generalized lasso problem \citep{gen_lasso_solution_path_2011}. Although this reformulation  does not apply directly under the logistic  log-likelihood, it can be readily employed for its tangent minorizers $\ellbound_{\textsc{bl}}$,~$\ellbound_{\textsc{pg}}$,~and $\ellbound_{\textsc{pq}}$ in Sections~\ref{subsec:quadratic_bounds}--\ref{sec_PLQ}, which rely on $L_1$ and $L_2$ terms. Therefore, in logistic regression, the~above optimization can be effectively solved via an \textsc{mm} algorithm  \citep[][]{mm_Hunter2004},  that iteratively approximates  the log-likelihood in  \eqref{log_lik_log} with one of $\ellbound_{\textsc{bl}}$,  $\ellbound_{\textsc{pg}}$, or $\ellbound_{\textsc{pq}}$ (tangent to it at the most recent estimate of $\bbeta$), and then maximizes the selected minorizer via standard generalized lasso updates \citep{gen_lasso_solution_path_2011}. Unlike Newton-Raphson, \textsc{mm} is monotone in the objective function. This ensures numerical stability and global convergence \citep{mm_Wu2010}.

In the \textsc{mm} context, tighter bounds require fewer iterations to reach the optimum \citep[e.g.,][]{em_McLachlan1996}. As illustrated empirically in Table~\ref{table_genlasso_summary}, this translates into computational gains for the proposed \textsc{pq} minorizer over both \textsc{pg} and \textsc{bl}. See the Supplementary Material for details on the~\textsc{mm} algorithms and the associated computational costs, which further clarify the settings~where~the \textsc{pq} bound is expected to provide the more remarkable gains over  \textsc{pg} and \textsc{bl}.

\subsection{Variational Bayes approximation}
\label{varb}
\vspace{-1pt}

Variational inference approximates the intractable posterior $p(\bbeta \mid \by) \propto p(\bbeta)\prod_{i=1}^{n} p(y_i \mid \bbeta{,} \bx_i)$ by a simpler density $q(\bbeta)$, chosen to minimize the Kullback–Leibler divergence  $\textsc{kl}[q(\bbeta) \lVert p(\bbeta \mid \by)]$, within a tractable approximating family $\mathcal{Q}$ \citep[see, e.g.,][]{Ormerod_2010_VB, Blei_2017}. When $\prod_{i=1}^{n} p(y_i \mid \bbeta, \bx_i)$ is the likelihood of a logistic regression, routine implementations rely on zero-mean Gaussian priors with fixed covariance matrix $\bOmega_0$, i.e.,  $p(\bbeta) = \phi(\bbeta;\bOmega_0)$, and consider multivariate normal variational families, namely,  $\mathcal{Q}=\{q(\bbeta): q(\bbeta) = \phi(\bbeta-\bmu;\bOmega) \}$ \citep[e.g.,][]{Jaakkola_Jordan_2000,Durante_Rigon_2019}. Under these settings, variational inference reduces to finding those $\bmu$ and $\bOmega$ that minimize $\textsc{kl}[ \phi(\bbeta-\bmu;\bOmega) \lVert p(\bbeta \mid \by)]$, or, alternatively, maximize the $\textsc{elbo}[\phi(\bbeta-\bmu;\bOmega)]=\mathbb{E}_{\phi(\bbeta-\bmu;\bOmega)}[ \log \phi(\bbeta;\bOmega_0)+ \sum_{i=1}^n \log p(y_i \mid \bbeta, \bx_i) - \log \phi(\bbeta-\bmu;\bOmega) ]$ \citep[e.g.,][]{Blei_2017,Durante_Rigon_2019}, where \smash{$ \log p(y_i \mid \bbeta, \bx_i)=(y_i-0.5) \bx_i^\top \bbeta + h(\bx_i^\top \bbeta)$},~as defined below equation  \eqref{log_lik_log}. Similarly to penalized maximum likelihood in Section~\ref{pen_li}, in this setting the non-quadratic terms $h(\bx_i^\top \bbeta)$~in~the log-likelihood hinder tractable maximization of the \textsc{elbo}. Such an issue can be addressed by replacing each $h(\bx_i^\top \bbeta)$~with a simpler tangent minorizer $\hbound(\bx_i^\top\bbeta \mid \zeta_i)$, where $\hbound$ can be, e.g., one of \smash{$\hbound_{\textsc{bl}}$, $\hbound_{\textsc{pg}}$, or $\hbound_{\textsc{pq}}$}~studied~in~Sections~\ref{subsec:quadratic_bounds}--\ref{sec_PLQ}.~As detailed in the Supplementary Material, this yields \textsc{vb} schemes employing the  \textsc{mm} principle~to~iteratively~optimize analytic~minorizers of the $\textsc{elbo}$ with respect to $(\bmu,\bOmega)$ and $\{\zeta_i\}_{i=1}^n$,~via~tractable~updates~similar, in terms of cost and simplicity,~to~those~derived~by~\citet{Jaakkola_Jordan_2000}~for~\textsc{pg}.

Within \textsc{vb}, tighter bounds for the \textsc{elbo}  guarantee  improved posterior approximation \citep[][]{Ormerod_2010_VB, Blei_2017}.
This is consistent with Figure~\ref{figure_tvd_fields}, where the \textsc{pq}~minorizer yields more accurate approximations than  \textsc{bl} and~\textsc{pg}, without increasing the computational costs. This motivates extensive use of \textsc{pq} in variational~Bayes as an improved alternative to \textsc{bl} and~\textsc{pg}.

\subsection{Empirical results in a criminology application}\label{sec_application}
To illustrate the practical gains of the \textsc{pq} bound under the methods presented in Sections~\ref{pen_li}--\ref{varb}, we analyze motor-vehicle theft data from Portland (Oregon), shared in 2015 by the \textsc{usa}~National Institute of Justice.  The dataset consists of $n=704$ spatial locations in the city, each associated with a binary response indicating whether it belongs to a high risk zone based on the number of thefts recorded. To infer a spatial map for the probability of being~in~a~high risk zone at any given location in the city (beyond already observed ones), we~employ~a~spatial logistic regression with {\em finite element} bases \citep[e.g.,][]{Lindgren_2011_SPDE,  Sangalli_2013_PDE} placed on a fine grid of the Portland map.  This yields $p=3103$ predictors, whose effects can be regularized~via~the~penalty $P_\lambda(\bbeta)$ in  Section~\ref{pen_li} or the prior  $\bOmega_0$ in Section~\ref{varb}, to encourage smoothness in the estimated spatial field \citep[e.g.,][]{Sangalli_2013_PDE}. In particular, we define~$\bD$~in~Section~\ref{pen_li}~to~enforce~Laplacian regularization, and, coherently, we let $\bOmega_0^{-1}= \lambda \bD^\top \bD$ under the Bayesian~model within Section~\ref{varb} (with \smash{$\lambda=10^{-5}$} to induce a smooth solution). See the Supplementary Material~for details.

\renewcommand{\arraystretch}{1}

\begin{table}[b]
	\caption{Performance of the \textsc{mm} schemes  for penalized maximum likelihood estimation under each of  the three~minorizers: number of iterations to reach convergence, total runtime in seconds, and time gain of  \textsc{pq} over \textsc{bl} and \textsc{pg}. The results are reported for two implementations: solution path only for $\lambda$, and solution path  for ($\lambda, \alpha$). 
		\vspace{10pt}}{
	\begin{tabular}{ccccccc}
		\quad & \multicolumn{3}{c}{Solution path for $\lambda$ ($\alpha=0.8$)} \qquad & \multicolumn{3}{c}{Solution path for ($\lambda$,$\alpha$)} \\
& Iterations & Total runtime [s] & Time gain [\%] &  \quad Iterations & Total runtime [s] & Time gain [\%] \\
		\textsc{bl}  \qquad  \qquad & 1731 & 111.87 & 57.68 &  \quad 10421 & 733.74 & 54.74 \\
		\textsc{pg}  \qquad  \qquad  & 1036 &  66.50 & 28.81 &  \quad 6468 & 465.08 & 28.59 \\
		\textsc{pq}  \qquad  \qquad  &  752 &  47.34 &       &   \quad 4513 & 332.10 &       \\
	\end{tabular}}
	\label{table_genlasso_summary}
\end{table}

\renewcommand{\arraystretch}{1}

Table~\ref{table_genlasso_summary} summarizes the computational performance of the \textsc{mm} schemes for penalized maximum likelihood estimation of the parameters in the above model, under the three minorizers analyzed (see Section~\ref{pen_li}). Results refer to realistic implementations \citep[][]{friedman2010regularization} exploring the entire solution path at different values for the tuning parameters. For $\lambda$ we consider an equally-spaced grid ranging from $-6$ to $+1$ on the $\log_{10}(\lambda)$ scale, while $\alpha$~is~either~fixed~at~$0.8$, or is also allowed to vary from $0.05$ to $0.95$ with $0.15$ increments. Following routine implementations, the solution paths are obtained by initializing $\bbeta$ at $\bzero_p$ and then proceeding backward~from the highest penalty value via efficient warm-start procedures. Consistent with the discussion in Section \ref{pen_li}, Table~\ref{table_genlasso_summary} confirms that the improved tightness of the \textsc{pq} bound yields a noticeable reduction~in~the number of iterations to convergence, at comparable per-iteration cost (since~all three \textsc{mm}s have to deal with the generalized lasso term from $P_\lambda(\bbeta)$ in the maximization step).~This~translates~into~remarkable gains in the total runtimes, both in absolute and relative~terms.

\begin{figure}[t!]
	\centering
	\includegraphics[width=1\textwidth]{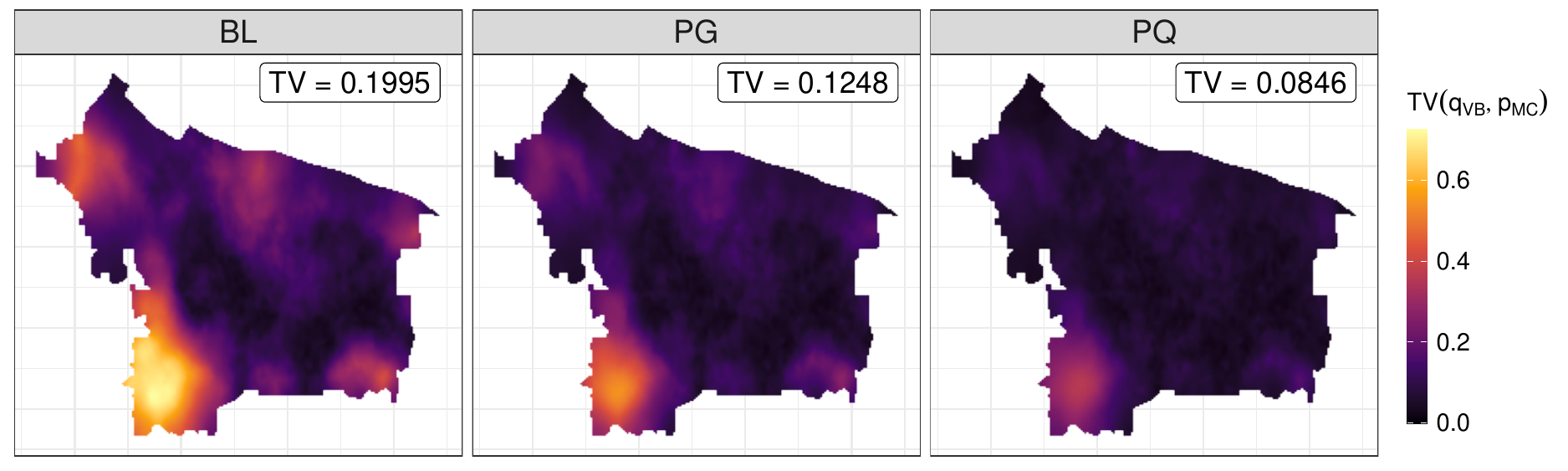}
		\vspace{-10pt}
	\caption{
		\textsc{vb} accuracy under the three bounds analyzed: \textsc{tv} distance between the actual posterior~on~the~spatial effects ${\bx^{\top}}\bbeta$ (for several configurations of $\bx$ covering a fine grid of the city map) and the corresponding   \textsc{vb} approximation obtained under the \textsc{bl},  \textsc{pg} and  \textsc{pq} minorizers. The posterior is estimated via Monte Carlo leveraging the Gibbs sampler of  \cite{Polson2012}. The top-right values within each panel correspond to the average of \textsc{tv} distances over the spatial~grid.}
	\label{figure_tvd_fields}
\end{figure} 

As shown in Figure~\ref{figure_tvd_fields}, the  \textsc{pq} tightness yields not only computational gains under \textsc{mm}, but~also accuracy improvements when employed within \textsc{vb} to approximate the target posterior~of~the~spatial effects over a fine grid of Portland's map (see Section~\ref{varb}). In particular, at each location we obtain a total variation (\textsc{tv}) distance between the approximate posterior derived under the~\textsc{pq} bound and the target one that is pointwise lower than those associated with \textsc{bl}~and~\textsc{pg}.~According to Figure~\ref{figure_tvd_fields}, these systematic gains are also evident in absolute terms (recall that $\textsc{tv} \in [0,1]$).

\vspace{55pt}
{\bf Acknowledgements.}  Niccol\'o Anceschi was partially supported by the United States National Institutes of Health (5R01ES027498-05, 5R01AI167850-03). Giacomo Zanella was funded by the European Union (ERC, PrSc-HDBayLe, project number 101076564). Cristian Castiglione, Tommaso Rigon and Daniele Durante have been funded by the European Union~(ERC,~NEMESIS, project number: 101116718). Views and opinions expressed are however those of the authors only and do not necessarily reflect those~of the European Union or the European Research Council Executive Agency.

\clearpage

\vspace{20pt}
\begin{center}
{\bf \huge Supplementary Material}
\end{center}

\vspace{45pt}

\appendix

\numberwithin{equation}{section}
\numberwithin{figure}{section}
\numberwithin{table}{section}
\newtheorem{lemma}{Lemma}[section]

\section{P\'{o}lya-Gamma data augmentation and PG bound}
\label{subsec_PG_DA}

The P\'{o}lya-Gamma data augmentation \citep{Polson2012} expands the logistic likelihood via a careful scale-mixture representation, introducing a P\'{o}lya-Gamma latent variable $z_i \in (0,\infty)$ for each statistical unit $i=1,\dots,n$, such that \smash{$(z_i \mid \bbeta) \simind \textsc{PG}(1,\bx_i^\top \bbeta)$}. In this way, conjugacy~between the Gaussian prior for $\bbeta$  and the augmented likelihood for $(\by,\bz \mid \bbeta)$ can be restored,~thereby yielding to a  Gaussian full-conditional distribution for $( \bbeta \mid \by,\bz)$ that  facilitates the implementation of a tractable Gibbs sampling scheme. More recently, \citet{Durante_Rigon_2019} leveraged on the same hierarchical representation to construct a mean-field variational approximation of the joint posterior $p(\bbeta, \bz \mid \by)$.
In doing so, the authors proved that, although originally developed by \citet{Jaakkola_Jordan_2000} under a purely mathematical argument, the \textsc{pg} bound for the logistic log-likelihood, i.e., \smash{$\ellbound_{\textsc{pg}} (\bbeta \mid \widetilde{\bbeta})=\sum_{i=1}^n [ (y_i -0.5) \bx_i^\top \bbeta + \hbound_{\textsc{pg}}(\bx_i^\top \bbeta\mid \bx_i^\top \widetilde{\bbeta}) ]$}, can be equivalently re-expressed as
\begin{align*}
 &   \ellbound_{\textsc{pg}} (\bbeta \mid \widetilde{\bbeta})  = \mathbb{E}_{p(\bz \mid \widetilde{\bbeta})} \left[ \log \frac{ p(\by, \bz \mid \bbeta)}{ p(\bz \mid \widetilde{\bbeta})} \right] = \sum\nolimits_{i=1}^n \mathbb{E}_{p(z_i \mid \widetilde{\bbeta})} \left[ \log \frac{ p(y_i, z_i \mid \bbeta)}{ p(z_i \mid \widetilde{\bbeta})} \right] \\
    &= \sum\nolimits_{i=1}^n [  (y_i -1/2)\bx_i^\top \bbeta-0.5\bx_i^\top \widetilde{\bbeta}  - 0.5\cdot w_{\textsc{pg}}(\bx_i^\top \widetilde{\bbeta}) \big( (\bx_i^\top \bbeta)^2 -(\bx_i^\top \widetilde{\bbeta})^2 \big)-\log(1+\exp(-\bx_i^\top \widetilde{\bbeta} ))] \; .
\end{align*}
Indeed $w_{\textsc{pg}}(\bx_i^\top \widetilde{\bbeta}) = \tanh(\bx_i^\top \widetilde{\bbeta} /2 ) / (2 \; \bx_i^\top \widetilde{\bbeta})$
coincides with the expected value of the $\mbox{PG}(1,\bx_i^\top \widetilde{\bbeta}) $ random variable.

\vspace{10pt}

\section{Previously proposed piece-wise quadratic bounds}
\label{subsec_PLQ_bound}

\setcounter{figure}{0}
\setcounter{table}{0}
\setcounter{equation}{0}

The key aspect motivating interest in  tangent quadratic lower bounds for logistic log-likelihoods is the associated high tractability, both in the derivation of the bound itself and in its direct~maximization. However, such a tractability might come at the expense of a limited approximation~accuracy. 
As discussed in the main article, this has motivated several subsequent contributions~in the literature aimed at deriving more accurate lower bounds for logistic log-likelihoods, with a specific focus on piece-wise quadratic minorizers \citep[see, e.g.,][]{Marlin_ICML2011_Piecewise_Q_bounds,khan12_multinomial,Ermis_Bouchard_2014}.

Within the above framework, a seminal contribution is the one by \citet{Marlin_ICML2011_Piecewise_Q_bounds}~who proposed the use of a fixed minimax-optimal piece-wise quadratic bound among all the possible piece-wise quadratic (\textsc{r-pq}) tangent minorizers of the logistic log-likelihood defined as
\begin{equation*}
    h_{\textsc{r-pq}}(r ; R) = \sum\nolimits_{s=1}^R (a_s + b_s r +
  c_s r^2) \cdot \bbone(r \in [t_{s-1},t_s)).
\end{equation*}
The authors consider the number of disjoint intervals $R$ composing the domain of the minorizing  function to be a principal tunable parameter, which regulates a trade-off between the accuracy and the complexity of the resulting approximation.
For an arbitrary number $R$ of intervals, the piece-wise quadratic bound is then constructed by solving numerically a minimax optimization problem both on the locations identifying the interval's separation and on the local coefficients of the quadratic contributions. 

Specifically, adapting the notation to our setting,  the minimax solution~is~obtained by solving
\begin{equation*}
\begin{split}
&\min\nolimits_{\{a_s,b_s,c_s,t_s\}}  \max\nolimits_{s=1,\dots,R} \max\nolimits_{r \in [t_{s-1},t_s)} \; \left[ h(r) - h_{\textsc{r-pq}}(r ; R) \right] \\
& \qquad    \begin{cases}
\;\; h(r) - ( a_s + b_s r +
  c_s r^2) \geq 0 &\quad \forall s=1,\dots,R, \; \forall r \in [t_{s-1},t_s), \\[3pt]
\;\; t_s - t_{s-1} > 0 &\quad \forall s=1,\dots,R, \\[3pt]
\;\; c_s \leq 0 &\quad \forall s=1,\dots,R,
\end{cases}
\end{split}
\end{equation*}
while further imposing bounded discrepancy from the target in each of the $R$ sets. The output of this numeric optimization was originally exploited within a generalized \textsc{em} algorithm to overcome the intractability of some logistic-Gaussian integrals, replacing the logistic log-likelihoods with fixed piece-wise quadratic bounds. Therefore, in this case, the construction of~the~fixed~bound~is separated from the learning phase of the inferential procedure. More specifically, such a bound is treated as a pre-computed approximation of an analytically intractable component of the model, whose accuracy is controlled via the cardinality of the underlying partitioning~of~the~domain.

Although the \textsc{pq} bound we propose in Section~\ref{sec_PLQ} is implicitly included in the general family of piece-wise quadratic tangent lower bounds for the logistic log-likelihood, such a novel minorizer substantially improves the one of \citet{Marlin_ICML2011_Piecewise_Q_bounds} in terms of tractability. First, \textsc{pq}  does~not requires solving an internal minimax optimization problem, but rather is available through an explicit analytical formulation which entails only a single splitting point for the domain of each likelihood contribution. While this avoids the need to learn the locations of the knots, we further restrict the degree of freedom by imposing the same curvature on both the resulting quadratic branches. In doing so, we overcome the need of imposing bounded discrepancy from the target, since the increased flexibility of the \textsc{pq} bound already provides a substantial accuracy gain~over the purely-quadratic minorizers by \citet{Bohning1988} and \citet{Jaakkola_Jordan_2000}. Finally, as illustrated in the main article, the quantities $\bzeta = (\zeta_1,\dots,\zeta_n)^\top$, which parameterize~the bound and its coefficients, can be learned adaptively as part of the inferential procedure instead of being pre-determined via a minimax optimization routine. Such an adaptive learning is inherent to state-of-the-art routines employing tangent minorizers (including quadratic ones), such as \textsc{em} and \textsc{mm} for maximum likelihood estimation \citep[e.g.,][]{mm_Hunter2004, mm_Wu2010}, and \textsc{cavi} for mean-field variational inference \citep[e.g.,][]{Jaakkola_Jordan_2000,bishop2006pattern}. 

\vspace{3pt}

\section{Technical lemmas and proofs}\label{app:proofs}

\subsection{Technical lemma}
Before proceeding with the proofs of the main results in the article, let us state and prove~a~technical lemma which is useful for proving Proposition \ref{prop:PLQ_dom_PG}.

\begin{lemma}\label{lemma:third_der}
Calling $\mathbb{R}_+:=[0,\infty)$, let $f:\mathbb{R}_+\to\mathbb{R}_+$ be a $C^1$, strictly concave function with $f(0)<0$, $f(\zeta)=0$ and
\smash{$
\int_{0}^\zeta f(s)ds=0
$}
for some $\zeta>0$.
Then
\smash{$
\int_{0}^r f(s)ds\leq 0
$}
for all $s\in\mathbb{R}_+$.
\end{lemma}
\begin{proof}[Proof of Lemma \ref{lemma:third_der}]
By concavity of $f$, we have
$$
0
=
\int_{0}^\zeta f(r)dr
<
\int_{0}^\zeta f(\zeta)+f'(\zeta)(r-\zeta)dr
=
-f'(\zeta)\zeta^2/2\,,
$$
which implies $f'(\zeta)<0$. 
Thus, again by concavity of $f$, we obtain
\begin{align*}
\int_{\zeta}^r f(s)ds
<
&
\int_{\zeta}^r
f(\zeta)+f'(\zeta)(s-\zeta)ds
=
f'(\zeta)(r-\zeta)^2/2<0,
\end{align*}
for all $r> \zeta$. Consider now $r\leq \zeta$.
Since $f$ is strictly concave it has at most two zeros, one of which is at $\zeta$.
Since $f(\zeta)=0$ and $f'(\zeta)<0$, then $f$ is positive in a left neighborhood of $\zeta$.
Thus, by $f(0)<0$ and the continuity of $f$,  we have that $f$ has a second zero in $(0,\zeta)$, which we denote as $r_0$.
Thus, the function $g(r):=\int_{0}^r f(s)ds$ is strictly decreasing in $(0,r_0)$, strictly increasing in $(r_0,\zeta)$ and satisfies $g(0)=g(\zeta)=0$, which imply that $g(r)<0$ for $r\in(0,\zeta)$.
\end{proof}

\subsection{Proofs}
\begin{proof}[Proof of Lemma \ref{Lemma_optimality_pg_1D_space}]
Considering any element $\hbound_{\textsc{q}} \in \mathcal{H}_{\textsc{q}}$, the tangency conditions on $ \hbound_{\textsc{q}} $ imply
\begin{equation}
\label{eq_proof_lemma1}
\begin{split}
    h(\zeta) &= \hbound_{\textsc{q}}(\zeta \mid \zeta) = a(\zeta) + b(\zeta)\zeta + c(\zeta)\zeta^2, \\
    h'(\zeta) &= \hbound_{\textsc{q}}'(\zeta \mid \zeta) = b(\zeta) + 2 c(\zeta)\zeta,
\end{split}
\end{equation}
for any $\zeta \in \mathbb{R}$. Here we  have exploited the differentiability of both the target function and the quadratic minorizers, where $ h'(r) = \partial h(r) / \partial r$.
Additionally, the symmetry of the former gives
\begin{align*}
    h(\zeta) &= h(-\zeta) \geq \hbound_{\textsc{q}}(-\zeta \mid \zeta) = a(\zeta) - b(\zeta)\zeta + c(\zeta)\zeta^2 = h(\zeta) - 2 b(\zeta)\zeta, 
\end{align*}
which implies $b(\zeta)\zeta \geq 0$.
After multiplying by $\zeta$ both sides in the second line of \eqref{eq_proof_lemma1}, the above inequality gives
$2 c(\zeta) \leq h'(\zeta)/\zeta = - w_{\textsc{pg}}(\zeta)$.
Eliciting the (negative) curvature of 
$\hbound_{\textsc{q}}(\,\bigcdot \mid \zeta)$ as $ w_{\textsc{q}}(\zeta) = - 2 c(\zeta)$, one equivalently has $w_{\textsc{q}}(\zeta) \geq w_{\textsc{pg}}(\zeta)$.
To conclude the proof, it is sufficient to recall that the tangent minorization conditions always constrain the constant and linear term of any purely quadratic bound. 
Accordingly, any element of $\mathcal{H}_{\textsc{q}}$ can be rewritten as 
\begin{align*}
    \hbound_{\textsc{q}}(r \mid \zeta) &= h(\zeta) + h'(\zeta) (r-\zeta) - \frac{1}{2} w_{\textsc{q}}(\zeta) (r-\zeta)^2 ,
\end{align*}
including $\hbound_{\textsc{pg}}$.
This implies that $\hbound_{\textsc{q}}(r \mid \zeta) - \hbound_{\textsc{pg}}(r \mid \zeta) = \frac{1}{2} \big(w_{\textsc{pg}}(\zeta) - w_{\textsc{q}}(\zeta) \big) (r-\zeta)^2 \leq 0$. The point-wise optimality of $\hbound_{\textsc{
pg}}$ over any alternative quadratic bound is  illustrated in Figure~\ref{figure_quadratic_lb}.
\end{proof}

\begin{figure}[t]
\centering
    \includegraphics[width=1\textwidth]{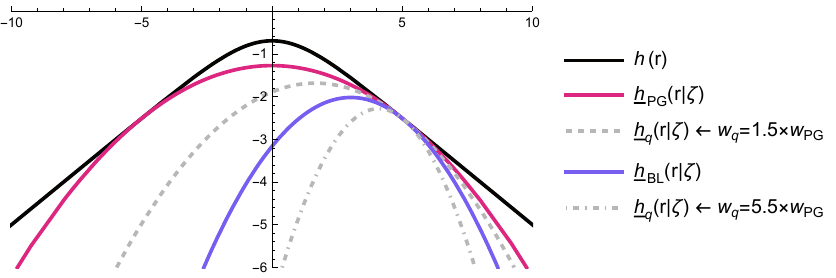}
    \caption{Comparison between different quadratic bounds for $
    h(r)$, tangent to the latter in $\zeta=5$. $\hbound_{\textsc{pg}}(r \mid \zeta)$ and $\hbound_{\textsc{bl}}(r \mid \zeta)$ are described in Section~\ref{subsec:quadratic_bounds} of the main article, while the two quadratic lower bounds $\hbound_{\textsc{q}}(r \mid \zeta)$ corresponds respectively to $w_{\textsc{q}}(\zeta) = 1.5 \cdot w_{\textsc{pg}}(\zeta)$ and $w_{\textsc{q}}(\zeta) = 5.5 \cdot w_{\textsc{pg}}(\zeta)$.}
    \label{figure_quadratic_lb}
\end{figure}

\vspace{5pt}
\begin{proof}[Proof of Corollary \ref{Corollary_optimality_pg_beta_space}]
First notice that by replacing each $\hbound(\bx_i^\top \bbeta\mid  \bx_i^\top \widetilde{\bbeta})$ in \eqref{eq:glob_bound} with any bound $\hbound_{\textsc{q},i} \in \mathcal{H}_{\textsc{q}}$ yields~a~minorizer $\ellbound_{\textsc{q}}$ for $\ell$ that belongs to $\mathcal{G}_{\textsc{q}}$. Therefore, also  $\ellbound_{\textsc{pg}} \in \mathcal{G}_{\textsc{q}}$ provided that $\hbound_{\textsc{pg}} \in \mathcal{H}_{\textsc{q}}$. Since \smash{$\sum\nolimits_{i=1}^n (y_i -0.5) \bx_i^\top \bbeta $} is in common to all bounds in \smash{$\mathcal{G}_{\textsc{q}}$}, to prove the optimality of \textsc{pg} within \smash{$\mathcal{G}_{\textsc{q}}$}, it suffices to compare the generic \smash{$\sum\nolimits_{i=1}^n \hbound_{\textsc{q},i}(\bx_i^\top \bbeta\mid  \bx_i^\top \widetilde{\bbeta})$} with  \smash{$\sum\nolimits_{i=1}^n \hbound_{\textsc{pg}}(\bx_i^\top \bbeta\mid  \bx_i^\top \widetilde{\bbeta})$}. By contradiction, assume there is an \smash{$\widehat{\ellbound}_{\textsc{q}} \in \mathcal{G}_{\textsc{q}}$} and \smash{$\widehat{\bbeta}\in\mathbb{R}^p$} such that there exists a $\bbeta^*$ for which \smash{$\ell(\bbeta^*) \geq \widehat{\ellbound}_{\textsc{q}}(\bbeta^* \mid \widehat{\bbeta}) > \ellbound_{\textsc{pg}}(\bbeta^* \mid \widehat{\bbeta})$}.
This would mean that there exists at least one $i^* \in \{1,\dots,n \}$ such that $h(\bx_{i^*}^\top \bbeta^*) \geq \hbound_{\textsc{q},i^*}(\bx_{i^*}^\top \bbeta^* \mid \bx_{i^*}^\top \widehat{\bbeta}) > \hbound_{\textsc{pg}}(\bx_{i^*}^\top \bbeta^* \mid \bx_{i^*}^\top \widehat{\bbeta})$. The latter relation is not possible by Lemma \ref{Lemma_optimality_pg_1D_space}, thereby proving Corollary \ref{Corollary_optimality_pg_beta_space}.
\end{proof}

\vspace{5pt}
\begin{proof}[Proof of Proposition \ref{prop:PLQ_dom_PG}]
Let us start by proving $h(r) \geq \hbound_{\textsc{pq}}(r \mid \zeta)$, $\forall \; r, \zeta \in \mathbb{R}$.
For $\zeta=0$,~the $\textsc{pq}$ and \textsc{pg} bounds coincide, thus the proposition is clearly satisfied. In addition, notice that by the symmetry of the target function and of the newly-proposed  \textsc{pq} bound (see the right~panel~of~Figure~\ref{figure_plq_r2} in the main article)  it suffices to focus on $\zeta>0$ and $r>0$.

For ease of exposition let us consider a translated version
\begin{equation*}
\begin{aligned}
    \widehat{h}(r) &=h(r) - h(0) = - \log \cosh (r/2), \\ \widehat{\hbound}_{\textsc{pq}}(r \mid \zeta)&= \hbound_{\textsc{pq}}(r \mid \zeta) - h(0), 
\end{aligned}
\end{equation*} 
of the target function and of its minorizer, so that $\widehat{h}(0)=0$. Under these settings, $\widehat{\hbound}_{\textsc{pq}}(r \mid \zeta)$ is a polynomial of order two in  $r>0$, such that
\begin{equation*}
 \widehat{\hbound}_{\textsc{pq}}(0)=0, \quad \widehat{\hbound}_{\textsc{pq}}(\zeta)=h(\zeta), \quad \widehat{\hbound'}_{\textsc{pq}}(0)<0,\quad \widehat{\hbound'}_{\textsc{pq}}(\zeta)=\widehat{h}'(\zeta).
\end{equation*}
Moreover, $\widehat{h}'(0)=0$ and  $\widehat{h}'''(r) = 1/4 \operatorname{sech}^2(r/2) \operatorname{tanh}(r/2) >0$ for all $r>0$, which~crucially~implies that \smash{$\widehat{h}$} (and hence $h$) is strictly concave. Leveraging these results, our goal is to prove that \smash{$\widehat{\hbound}_{\textsc{pq}}(r \mid \zeta) - \widehat{h}(r) < 0$} for all $r>0$. To this end, write
\begin{equation*}
\widehat{\hbound}_{\textsc{pq}}(r \mid \zeta) - \widehat{h}(r)=\int_0^r f(s)ds,    
\end{equation*}
with $f(s)=\widehat{\hbound'}_{\textsc{pq}}(s  \mid \zeta) - \widehat{h}'(s)$. Then, the desired result holds  by Lemma \ref{lemma:third_der}, whose assumptions are satisfied because \smash{$f''(s)=-\widehat{h}'''(s)<0$} for every $s>0$. As discussed above, by symmetry, the same inequality is verified for $r<0$, thereby proving $h(r) \geq \hbound_{\textsc{pq}}(r \mid \zeta)$, $\forall \; r, \zeta \in \mathbb{R}$.

To conclude the proof of Proposition \ref{prop:PLQ_dom_PG}, let us now turn the attention to proving the inequality $\hbound_{\textsc{pq}}(r \mid \zeta) \geq \hbound_{\textsc{pg}}(r \mid \zeta)$, $\forall \; r, \zeta \in \mathbb{R}$. To this end, we leverage the inequality \smash{${\lvert r \rvert \leq \frac{1}{2} \big( r^2/\lvert \zeta \rvert + \lvert \zeta \rvert \big)}$}, often employed in the \textsc{mm} literature \citep{mm_Hunter2004,mm_Wu2010}.
To exploit this result, first recall that the tangency condition for the minorizer in the transformed space requires that the quantities
\begin{align*}
    \tilde{\hbound}_{\textsc{pg}}'(\varphi \mid \varphi) = -\frac{1}{2} \widetilde{w}_{\textsc{pg}}(\varphi)
    \quad \mbox{and} \quad
    \tilde{\hbound}_{\textsc{pq}}'(\varphi \mid \varphi) = -\frac{1}{2}\widetilde{w}_{\textsc{pq}}(\varphi) -\frac{1}{2} \frac{1}{\sqrt{\varphi}} \widetilde{\nu}_{\textsc{pq}}(\varphi),
\end{align*}
are both equal to $\tilde{h}'(\varphi)$. 
In the original space, this implies $ w_{\textsc{pg}}(\zeta) = w_{\textsc{pq}}(\zeta) + \nu_{\textsc{pq}}(\zeta) / \lvert \zeta \rvert $. 

As such
\begin{align*}
    h(r) \geq 
    \hbound_{\textsc{pq}}(r \mid \zeta)
   &= h(\zeta) - \frac{1}{2} w_{\textsc{pq}}(\zeta) (r^2-\zeta^2) - \nu_{\textsc{pq}}(\zeta)(\lvert r \rvert- \lvert \zeta \rvert)\\
   &  \geq h(\zeta) - \frac{1}{2} w_{\textsc{pq}}(\zeta) (r^2-\zeta^2) - \frac{1}{2} \nu_{\textsc{pq}}(\zeta) \frac{1}{\lvert \zeta \rvert }( r^2 - \zeta^2)\\
   & = h(\zeta) - \frac{1}{2} w_{\textsc{pg}}(\zeta)( r^2 - \zeta^2) =  \hbound_{\textsc{pg}}(r \mid \zeta) \;.
\end{align*}
The second line is obtained by applying the previously-mentioned inequality to the absolute value function. The third line is a consequence of $ w_{\textsc{pg}}(\zeta) = w_{\textsc{pq}}(\zeta) + \nu_{\textsc{pq}}(\zeta) / \lvert \zeta \rvert $.
\end{proof}

\vspace{16pt}
\begin{proof}[Proof of Lemma \ref{Lemma_optimality_plq}]
Without loss of generality, let us rewrite here every element $\hbound_{\textsc{s}} \in \mathcal{H}_{\textsc{s}}$ as
\begin{align*}
  \hbound_{\textsc{s}}(r \mid \zeta) = h(\zeta) + b_{\textsc{s}}(\zeta) (r-\zeta) - \frac{1}{2} w_{\textsc{s}}(\zeta) (r^2-\zeta^2) - \nu_{\textsc{s}}(\zeta)(\lvert r \rvert - \lvert \zeta \rvert ) \; , 
\end{align*}
where $w_{\textsc{s}}(\zeta) =- 2 c_{\textsc{s}}(\zeta)$, $\nu_{\textsc{s}}(\zeta) = - d_{\textsc{s}}(\zeta)$, and the intercept is fixed by condition $\hbound_{\textsc{s}}(\zeta \mid \zeta)=h(\zeta)$.

The above formula facilitates comparison with the formulation of the \textsc{pq} bound  in \eqref{plq_original_space}.
For ease of exposition, we additionally consider again the translated version \smash{$\widehat{h}(r)$} of the target function~$h(r)$ 
\begin{equation*}
\begin{aligned}
    \widehat{h}(r) &= h(r) - h(0) = - \log \cosh (r/2) \; ,
\end{aligned}
\end{equation*}
so that $\widehat{h}(0)=0$, while preserving the symmetry $\widehat{h}(-r) = \widehat{h}(r)$.
Accordingly, any $ \hbound_{\textsc{s}} \in \mathcal{H}_{\textsc{s}}$ results in a proper minorizer for the adjusted target by simply applying the same rigid translation~$\smash{\widehat{\hbound}_{\textsc{s}}}(r \mid \zeta) = \hbound_{\textsc{s}}(r \mid \zeta) -h(0)$, so that
\begin{equation*}
\left\{
\begin{aligned}
        \, & \widehat{h}(r) \geq \widehat{\hbound}_{\textsc{s}}(r \mid \zeta) = \widehat{h}(\zeta) + b_{\textsc{s}}(\zeta) (r-\zeta) - \frac{1}{2} w_{\textsc{s}}(\zeta) (r^2-\zeta^2) - \nu_{\textsc{s}}(\zeta)(\lvert r \rvert - \lvert \zeta \rvert ), \\[3pt]
       \, & \widehat{h}(\zeta) = \widehat{\hbound}_{\textsc{s}}(\zeta \mid \zeta ).
\end{aligned} \right.
\end{equation*}
In particular, the minorization requirement at the specific locations $r=- \zeta$ and $r=0$ implies
\begin{equation*}
\left\{
\begin{aligned}
    \, & \widehat{h}(\zeta) =  \widehat{h}(-\zeta) \geq \widehat{\hbound}_{\textsc{s}}(-\zeta \mid \zeta ) = \widehat{h}(\zeta) - 2 b_{\textsc{s}}(\zeta) \zeta, \\[3pt]
    \, & 0 = \widehat{h}(0) \geq \widehat{\hbound}_{\textsc{s}}(0 \mid \zeta) = \widehat{h}(\zeta) - b_{\textsc{s}}(\zeta) \zeta + \frac{1}{2} w_{\textsc{s}}(\zeta) \zeta^2 + \nu_{\textsc{s}}(\zeta) \lvert \zeta \rvert.
\end{aligned} \right.
\end{equation*}
Assume that $\zeta \neq 0$, so that the minorizers are differentiable at $\zeta$.
As such, the tangent minorization requirement further gives
\begin{equation*}
    \widehat{h}'(\zeta) = \widehat{\hbound}_{\textsc{s}}~\hspace{-5pt}'\hspace{1pt}(\zeta \mid \zeta) = b_{\textsc{s}}(\zeta) - w_{\textsc{s}}(\zeta) \zeta - \nu_{\textsc{s}}(\zeta) \, \sign(\zeta),
\end{equation*}
where $\sign(\cdot)$ is the sign function.

Combining the above result with the previous conditions, we have that
\begin{equation*}
\left\{ 
\begin{aligned}
    \, & b_{\textsc{s}}(\zeta) - \nu_{\textsc{s}}(\zeta) \, \sign(\zeta) = \widehat{h}'(\zeta) + w_{\textsc{s}}(\zeta) \zeta,
    \\[3pt]
    \, & w_{\textsc{s}}(\zeta) \geq \smash{2 [\widehat{h}(\zeta) - \widehat{h}'(\zeta) \, \zeta ]/\zeta^2}, \\[3pt]
    \, & b_{\textsc{s}}(\zeta) \zeta \geq 0.
\end{aligned} \right.
\end{equation*}
Conversely, we recall that the \textsc{pq} bound $\widehat{\hbound}_{\textsc{pq}}(r \mid \zeta)$ arises by imposing both $b_{s}(\zeta) = 0$ --- which implies symmetry with respect to the origin --- and also \smash{$\widehat{h}(0) = \widehat{\hbound}_{\textsc{pq}}(0 \mid \zeta)$}. 

These two restrictions translate into
\begin{equation*}
\begin{split}
    w_{\textsc{pq}}(\zeta) = \frac{2}{\zeta^2} \left[\widehat{h}(\zeta) - \widehat{h}'(\zeta) \, \zeta \right], \qquad \qquad
    \nu_{\textsc{pq}}(\zeta) = \frac{1}{\vert \zeta \vert} \left[ \widehat{h}'(\zeta) \, \zeta - 2 \, \widehat{h}(\zeta) \right],
\end{split}
\end{equation*}
which, in particular, means that $w_{\textsc{pq}}(\zeta) \leq w_{\textsc{s}}(\zeta) $.
Assume now that $\zeta >0$. If $r>0$, then
\begin{equation*}
\begin{split}
    \widehat{\hbound}_{\textsc{pq}}(r \mid \zeta) - \widehat{\hbound}_{\textsc{s}}(r \mid \zeta)
    &= - [\nu_{\textsc{pq}}(\zeta) - \nu_{\textsc{s}}(\zeta) + b_{\textsc{s}}(\zeta)] (r - \zeta) - \frac{1}{2} [w_{\textsc{pq}}(\zeta) - w_{\textsc{s}}(\zeta) ] (r^2 - \zeta^2) \\
    &=  [ w_{\textsc{pq}}(\zeta) - w_{\textsc{s}}(\zeta) ] \zeta (r - \zeta) - \frac{1}{2} [ w_{\textsc{pq}}(\zeta) - w_{\textsc{s}}(\zeta)] (r^2 - \zeta^2) \\
    &= - \frac{1}{2} [w_{\textsc{pq}}(\zeta) - w_{\textsc{s}}(\zeta)] (r - \zeta)^2 \geq 0.
\end{split}
\end{equation*}
On the other hand, if $r<0$ then
\begin{align*}
    \widehat{\hbound}_{\textsc{pq}}(r \mid \zeta) - \widehat{\hbound}_{\textsc{s}}(r \mid \zeta) &= [ \widehat{\hbound}_{\textsc{pq}}(-r \mid \zeta) - \widehat{\hbound}_{\textsc{s}}(-r \mid \zeta)] - 2 \, b_{\textsc{s}}(\zeta) \, r \\
    &= [ \widehat{\hbound}_{\textsc{pq}}( \vert r \vert \mid \zeta) - \widehat{\hbound}_{\textsc{s}}( \vert r \vert \mid \zeta) ]- 2 \, b_{\textsc{s}}(\zeta) \, r \geq 0.
\end{align*}
Indeed, the first term is non-negative thanks to the previous equation, whereas the second one is non-negative because $r<0$ and $b_{\textsc{s}}(\zeta) \geq 0$, since $\zeta >0$.
Finally, for $r=0$ it holds by definition that $\widehat{\hbound}_{\textsc{pq}}(0 \mid \zeta) = \widehat{h}(0) \geq \widehat{\hbound}_{\textsc{s}}(0 \mid \zeta)$, concluding that $\widehat{\hbound}_{\textsc{pq}}(r \mid \zeta) \geq \widehat{\hbound}_{\textsc{s}}(r \mid \zeta)$ for any $r \in \mathbb{R}$. 

Recall that all the above derivations have been obtained for $\zeta>0$. Analogous results can be derived for the case $\zeta < 0$, which we hereby omit to avoid redundancies.
Conversely, for the case $\zeta = 0$ we can focus on the right and left limits of the difference quotient
\begin{equation*}
    \left\{ 
\begin{aligned}
    & \lim_{r \rightarrow 0^+} \frac{\widehat{\hbound}_{\textsc{s}}(r  \mid 0)-\widehat{\hbound}_{\textsc{s}}(0 \mid 0)}{r} = b_{\textsc{s}}(0) - \nu_{\textsc{s}}(0) \\
    &\lim_{r \rightarrow 0^-} \frac{\widehat{\hbound}_{\textsc{s}}(r  \mid 0)-\widehat{\hbound}_{\textsc{s}}(0 \mid 0)}{r} = b_{\textsc{s}}(0) + \nu_{\textsc{s}}(0) 
\end{aligned}
\right.
\end{equation*}
which have  to be both equal to $\widehat{h}'(0)=0$. In particular, this implies that $\nu_{\textsc{s}}(0)=0$, which~amounts to dropping the piece-wise linear term. 
As consequence, one has $\hbound_{\textsc{s}}(\,\bigcdot \mid 0) \in \mathcal{H}_{\textsc{s}} \cap \mathcal{H}_{\textsc{q}}$, ensuring~the optimality of \smash{$\widehat{\hbound}_{\textsc{pq}}$}. 
Recall that $\hbound_{\textsc{pq}}(\,\bigcdot \mid 0) = \hbound_{\textsc{pg}}(\,\bigcdot \mid 0) = \hbound_{\textsc{bl}}(\,\bigcdot \mid 0)$.
\end{proof}

\vspace{15pt}

\section{ Penalized maximum likelihood estimation via tangent minorizers}\label{appendix_optim}

\setcounter{figure}{0}
\setcounter{table}{0}
\setcounter{equation}{0}

The generalized  elastic-net penalty $\lambda [ \alpha \lVert \bD \bbeta \rVert_1+0.5{\cdot}(1-\alpha) \lVert \bD \bbeta \rVert_2^2 ]$  studied in Section~\ref{pen_li}~of~the  article includes, as a special case, ridge \smash{$\lambda \lVert \bbeta \rVert_2^2$}, lasso $\lambda \lVert \bbeta \rVert_1$, generalized ridge \smash{$\lambda \lVert \bD \bbeta \rVert_2^2$}, and generalized lasso $\lambda \lVert \bD \bbeta \rVert_1$. As such, it provides a comprehensive class that encompasses the most~widely employed regularizations within modern applications, including in mixed effects~modeling, nonparametric smoothing, spatial regression, wavelet signal extraction, mixtures of experts, isotonic regression, trend filtering and others \citep[see, e.g.,][]{fused_lasso_2005,elastic_net_paper, lin2006component,gen_lasso_solution_path_2011,zhao2012wavelet,Sangalli_2013_PDE,tibshirani2014adaptive,pastukhov2024fused,helwig2025versatile,javanmard2025prediction}. 
In the above penalty, the parameter~$\lambda \in \mathbb{R}_+$ determines the overall strength of the regularization, $\alpha \in [0,1]$ controls~the relative magnitude of the $L_1$ and $L_2$ norm contributions, while the $m \times p$ matrix $\bD$ defines the~$m$~directions subject to penalization. 
In the following, we provide the details of \textsc{mm} schemes for maximization of logistic log-likelihoods under this general penalty, leveraging~the~three~minorizers~analysed.

Let us first focus on the newly-proposed \textsc{pq}  bound. As discussed in Section~\ref{pen_li}, at the generic iteration $(t+1)$, the \textsc{mm} scheme minorizes the logistic log-likelihood with generalized elastic-net penalty through the \textsc{pq} bound, tangent to it at the current estimate of $\bbeta$ from iteration $(t)$,~and then updates such an estimate by maximizing the resulting tangent minorizer. Leveraging the generalized lasso representation in \eqref{eq_plq_bound_beta_intercept} of the \textsc{pq}  bound, and collecting the generalized ridge term of the penalty within the quadratic component, the resulting maximization problem becomes
\begin{equation}\label{g_l_PQ_MM}
	\underset{\bbeta \in \mathbb{R}^p}{\text{argmax}} \left\{ \; -\frac{1}{2} \bbeta^\top \mathbfcal{Q}_{\textsc{pq}}^{(t)} \bbeta + 
	\bbeta^\top \br_{\textsc{pq}} + \lambda \alpha \| \mathbfcal{D}_{\textsc{pq}}^{(t)} \bbeta \|_1 \; \right\},
\end{equation}
with \smash{$\mathbfcal{Q}_{\textsc{pq}}^{(t)} = \bX^\top \bW_{\textsc{pq}}^{(t)} \bX {+} \lambda (1-\alpha) \bD^\top \bD$}, \smash{$\br_{\textsc{pq}}= \bX^\top (\by - 0.5 {\cdot} {\bf 1}_n)$}, {$\mathbfcal{D}_{\textsc{pq}}^{(t)} =  [ (1/\lambda \alpha)\bX^{\top}\bN_{\textsc{pq}}^{(t)}{,} \bD^{\top} ]^{\top}$}, and \smash{$\bW_{\textsc{pq}}^{(t)} $} defined as in Section~\ref{sec_plq_gen_lasso}, after replacing \smash{$\tilde{\bbeta}$} with  \smash{${\bbeta}^{(t)}$}. Although the above optimization does not admit a closed-form solution due to the $L_1$ term \smash{$ \lambda \alpha \| \mathbfcal{D}_{\textsc{pq}}^{(t)} \bbeta \|_1$},~it~can~be~crucially~regarded as an instance of the standard generalized lasso  problem \citep{gen_lasso_solution_path_2011, gen_lasso_dual_path_2016}, which can be efficiently solved through either quadratic programming~\citep{Goldfarb1983_QP,Nesterov_Nemirovskii_1994_book_QP} or the alternating direction method of multipliers (\textsc{admm}) \citep{Boyd_2011_ADMM, Zhu_2017_ADMM}.
Here we opt for the \textsc{admm} scheme, which is well suited~for high-dimensional optimization,  leverages sparsity efficiently, and can be further accelerated via warm-start initialization and preconditioning. To this end, the maximization~\eqref{g_l_PQ_MM} at every step of  \textsc{mm} can be equivalently reformulated by introducing a set of auxiliary variables $\bz \in \mathbb{R}^{n+m}$~that~yield the following constrained optimization problem
\begin{equation*}
	\underset{\bbeta \in \mathbb{R}^p, \, \bz \in \mathbb{R}^{n+m}}{\text{argmax}} \left\{ \; -\frac{1}{2} \bbeta^\top \mathbfcal{Q}_{\textsc{pq}}^{(t)} \bbeta + 
	\bbeta^\top \br_{\textsc{pq}} - \lambda \alpha \|\bz\|_1 \;  \Bigm\vert \; \bz = \mathbfcal{D}_{\textsc{pq}}^{(t)} \bbeta \; \right\}.
\end{equation*}
Then, the \textsc{admm} scheme iteratively solves this constrained problem by searching~for~a~saddle~point of the associated augmented Lagrangian function \citep[see, e.g.,][]{Boyd_2011_ADMM, Zhu_2017_ADMM}, i.e.,
\begin{equation*}
	\mathcal{L}_{\textsc{pq}}^{(t)}(\bbeta, \bz, \bu) = -\frac{1}{2} \bbeta^\top \mathbfcal{Q}_{\textsc{pq}}^{(t)} \bbeta + 
	\bbeta^\top \br_{\textsc{pq}} - \lambda \alpha \|\bz\|_1 - \rho^{(t)} \bu^\top( \mathbfcal{D}_{\textsc{pq}}^{(t)} \bbeta - \bz) - \frac{1}{2} \rho^{(t)} \| \mathbfcal{D}_{\textsc{pq}}^{(t)} \bbeta - \bz \|_2^2,
\end{equation*}
where $\bu \in \mathbb{R}^{n+m}$ denotes a vector of (scaled) Lagrange multipliers and $\rho^{(t)} > 0$ is a regularization parameter penalizing the constraint violations, which are quantified by the $L_2$ term \smash{$\| \mathbfcal{D}_{\textsc{pq}}^{(t)} \bbeta - \bz \|_2^2$}.
Denoting by $(k)$ the inner iteration counter and initializing $\bbeta^{(t,0)} = \bbeta^{(t)}$, the \textsc{admm} cycles until convergence over the following closed-form updates for $\bbeta$, $\bz$ and $\bu$:
\begin{eqnarray}\label{inn_ADMM}
\begin{split}
	\bbeta^{(t,k+1)} &= (\mathbfcal{Q}_{\textsc{pq}}^{(t)} + \rho^{(t)} \mathbfcal{D}_{\textsc{pq}}^{(t)\top} \mathbfcal{D}_{\textsc{pq}}^{(t)} )^{-1} ( \br_{\textsc{pq}} + \mathbfcal{D}_{\textsc{pq}}^{(t)} (\bz^{(k)} - \bu^{(k)}) ), \\
	\bz^{(k+1)} &= \mathcal{S}_{\lambda \alpha / \rho^{(t)}} (\mathbfcal{D}_{\textsc{pq}}^{(t)} \bbeta^{(t,k+1)} + \bu^{(k)} ), \\
	\bu^{(k+1)} &= \bu^{(k)} + \mathbfcal{D}_{\textsc{pq}}^{(t)} \bbeta^{(t,k+1)} - \bz^{(k+1)},
\end{split}
\end{eqnarray}
where $\mathcal{S}_\delta(r)$ is the so-called soft-thresholding operator \citep{Hastie2015_book} defined as
\begin{equation}\label{soft_threshold}
	\mathcal{S}_\delta(r) = 
	\sign(r) (\vert r \vert - \delta)_{+} = 
	\begin{cases}
		\, r  - \delta & \quad \text{if} \;\;  r >0 \;\; \text{and} \;\; \delta<\vert r \vert, \\
		\, 0 & \quad \text{if} \;\; \delta \geq \vert r \vert, \\
		\, r  + \delta & \quad \text{if} \;\;  r <0 \;\; \text{and} \;\; \delta<\vert r \vert.
	\end{cases}
\end{equation}
At convergence of the inner \textsc{admm} optimization, after $K$ iterations, we set $\bbeta^{(t+1)} = \bbeta^{(t,K)}$.~Regardless of the specific implementation, the computational bottleneck of this routine comes~from the calculation of \smash{$(\mathbfcal{Q}_{\textsc{pq}}^{(t)} + \rho^{(t)} \mathbfcal{D}_{\textsc{pq}}^{(t)\top} \mathbfcal{D}_{\textsc{pq}}^{(t)} )^{-1} ( \br_{\textsc{pq}} + \mathbfcal{D}_{\textsc{pq}}^{(t)} (\bz^{(k)} - \bu^{(k)})) $}. Crucially,  in this expression the inverse does not vary with $(k)$, and hence, a convenient decomposition can be pre-computed just once within  the inner \textsc{admm} optimization. Leveraging either the Cholesky decomposition~or Woodbury identity, depending on whether $p<n$ or not, the resulting computational cost of~the \textsc{admm} cycle (and hence of each \textsc{mm} step) becomes $\mathcal{O}(\min\{p^3,n^3\} + np^2)$. Notice~that~in~these calculations we omit~the $\mathcal{O}(Kp^2)$ cost of the matrix-vector products for the updates in \eqref{inn_ADMM},~since, in practice, the number $K$ of iterations to reach convergence is often negligible relative to $n$. In our implementation, the inner  \textsc{admm} convergence is assessed by monitoring the increment in the objective function in \eqref{g_l_PQ_MM},~and~stopping the routine when such an increment is below $10^{-9}$. For the convergence of the entire \textsc{mm} routine we consider, instead, a $10^{-8}$ threshold both on the absolute and relative increments of the~original~objective~function \smash{$\ell(\bbeta)-\lambda [ \alpha \lVert \bD \bbeta \rVert_1{+}0.5{\cdot}(1-\alpha) \lVert \bD \bbeta \rVert_2^2 ]$}, where $\ell(\bbeta)$ is the logistic log-likelihood.

The aforementioned derivations and computational considerations apply directly also to the~\textsc{mm} schemes relying on  \textsc{bl}  \citep{Bohning1988} and  \textsc{pg} \citep[][]{Jaakkola_Jordan_2000} minorizers. In particular, replacing the proposed \textsc{pq}  bound with either  \textsc{bl}  or \textsc{pg}, yield \textsc{mm} schemes~relying on the same steps, with \smash{$\mathbfcal{Q}_{\textsc{pq}}^{(t)}$, $\br_{\textsc{pq}}$ and $\mathbfcal{D}_{\textsc{pq}}^{(t)}$} replaced by the corresponding counterparts under \textsc{bl} and \textsc{pg}. For \textsc{pg}, these are 
\begin{equation*}
\mathbfcal{Q}_{\textsc{pg}}^{(t)} = \bX^\top \bW_{\textsc{pg}}^{(t)} \bX + \lambda (1-\alpha) \bD^\top \bD, \qquad \br_{\textsc{pg}}=\br_{\textsc{pq}}, \qquad \mathbfcal{D}_{\textsc{pg}}^{(t)} = \bD.
\end{equation*}
Conversely, for \textsc{bl}, we have 
\begin{equation*}
\mathbfcal{Q}_{\textsc{bl}}^{(t)} = 0.25 {\cdot} \bX^\top \bX + \lambda (1-\alpha) \bD^\top \bD, \qquad \br^{(t)}_{\textsc{bl}}=\bX^{\top}(\by-\boldsymbol{\pi}^{(t)}+0.25{\cdot}\bX\bbeta^{(t)}), \qquad \mathbfcal{D}_{\textsc{bl}}^{(t)} = \bD,
\end{equation*}
where $\boldsymbol{\pi}^{(t)}=[\pi(\bx^{\top}_1\bbeta^{(t)}), \ldots, \pi(\bx^{\top}_n\bbeta^{(t)})]^{\top}$. 

The main difference between the \textsc{mm} schemes induced by these two  minorizers and the one~derived for the proposed \textsc{pq} bound is that~in~the~\textsc{admm} maximization \textsc{bl} and \textsc{pg} only require $m$ auxiliary variables $\bz$ and Lagrange multipliers $\bu$, rather than $m+n$. This is because the quadratic minorization of the logistic log-likelihood operated by  \textsc{bl} and \textsc{pg} does not require $L_1$ contributions. Therefore, the only non-differentiable components are those arising from the $m$ generalized lasso terms in the penalty \smash{$\lambda [ \alpha \lVert \bD \bbeta \rVert_1+0.5{\cdot}(1-\alpha) \lVert \bD \bbeta \rVert_2^2 ]$}. Nonetheless, the overall complexity~of each iteration of the resulting \textsc{mm} is still $\mathcal{O}(\min\{p^3,n^3\} + \smash{np^2})$, since it is dominated by the update of $\bbeta$ in \eqref{inn_ADMM}, which has the same dimensions,~and~hence, computational cost as in the \textsc{pq} case. Convergence assessments for the two \textsc{mm} schemes based on the \textsc{bl} and \textsc{pg} minorizers~rely~on the same checks and thresholds as those adopted for \textsc{pq}. 

\vspace{22pt}
\section{Variational inference via tangent minorizers}

In Bayesian contexts, variational inference aims at approximating the target intractable posterior $p(\bbeta \mid \by) \propto p(\bbeta) \prod_{i=1}^{n} p(y_i \mid \bbeta, \bx_i)$ through a  density $q(\bbeta)$ within a simpler, pre-specified, family~$\mathcal{Q}$. Recalling e.g., \citet[][]{Blei_2017}, such an approximation $q(\bbeta)$ is formally derived by  solving the Kullback-Leibler (\textsc{kl}) minimization problem $\argmin_{q \in \mathcal{Q}} \,\textsc{kl}[ q(\bbeta) \parallel p(\bbeta \mid \by) ]$ or, equivalently, by maximizing the evidence lower bound (\textsc{elbo}), which is defined as 
\begin{eqnarray*}
\begin{split}
	\textsc{elbo}[q(\bbeta)] 
	& = \sum\nolimits_{i=1}^{n} \EX_{q} [\log p(y_i \mid \bbeta, \bx_i)] + \EX_{q}[\log \{ p(\bbeta) / q(\bbeta) \}] \\
	& = - \,\textsc{kl}[ q(\bbeta) \parallel p(\bbeta \mid \by) ] + \log p(\by) \leq \log p(\by),
\end{split}
\end{eqnarray*}
where $\EX_{q}[\,\bigcdot\,]$ is the variational expectation computed with respect to the approximating density $q(\bbeta)$, while \smash{$p(\by) = \int p(\bbeta) \prod_{i=1}^{n} p(y_i \mid \bbeta, \bx_i) \,d\bbeta$} is the (intractable) marginal likelihood.

Recalling Section~\ref{varb}, we focus here on variational inference for the coefficients $\bbeta$ in Bayesian logistic regression, where $\log p(y_i \mid \bbeta, \bx_i)$ is defined as in  \eqref{log_lik_log}, $p(\bbeta) = \phi(\bbeta; \bOmega_0)$ corresponds to a zero-mean Gaussian prior with covariance matrix $\bOmega_0$, while $ \mathcal{Q}$ denotes the family of multivariate normal approximating densities, i.e., $\mathcal{Q} = \{ q(\bbeta) : q(\bbeta) = \phi(\bbeta - \bmu; \bOmega) \}$. This setting aligns~with standard practice in routine implementations \citep[][]{Jaakkola_Jordan_2000, Durante_Rigon_2019} and yields the following expression for the \textsc{elbo}
\begin{equation}
	\label{logit_elbo}
	\begin{aligned}
		&\textsc{elbo}[\phi(\bbeta - \bmu; \bOmega)] = \\
		& \qquad \sum\nolimits_{i=1}^{n} \big((y_i - 0.5) \bx_i^\top \bmu + \EX_{q}[h(\bx_i^\top \bbeta)] \big)  - \frac{1}{2} \big( \bmu^\top \bOmega_0^{-1} \bmu + \trace(\bOmega_0^{-1} \bOmega) - \log|\bOmega| \big) + c,
	\end{aligned}
\end{equation}
where $\trace(\,\bigcdot\,)$ and $|\bigcdot|$ denote the trace and the determinant of a matrix, respectively, whereas $c$~is~a constant term not depending on $\bmu$ and $\bOmega$. 

Under \eqref{logit_elbo}, variational inference reduces to maximize the \textsc{elbo} with respect to the mean~vector $\bmu$ and the covariance matrix $ \bOmega$ of the Gaussian approximating density. However, such an \textsc{elbo} lacks a closed-form expression due to the  intractable integral $\EX_{q}[h(\bx_i^\top \bbeta)]$. A possible solution to address this issue is to replace $h(\bx_i^\top \bbeta)$ with a convenient tangent minorizer $\hbound(\bx_i^\top\bbeta \mid \zeta_i)$~under which $\EX_{q}[\hbound(\bx_i^\top\bbeta \mid \zeta_i)]$ can be computed analytically. This provides a closed-form lower bound $\textsc{elbo}[\phi(\bbeta - \bmu; \bOmega) \mid \bzeta]$, $\bzeta = (\zeta_1, \dots, \zeta_n)^\top \in \R^n$, of the the original  \textsc{elbo}, which can be used in place of \eqref{logit_elbo}  as a measure of approximation accuracy to be maximized with respect to $\bmu$,~$\bOmega$ and $\bzeta$ for obtaining a convenient variational approximation of the target posterior density $p(\bbeta \mid \by)$. Such a direction has been actively explored and successfully implemented in several contributions~leveraging the \textsc{bl} and \textsc{pg} bounds \citep[see, e.g.,][]{Jaakkola_Jordan_2000, bishop2003bayesian, Marlin_ICML2011_Piecewise_Q_bounds,Ren2011LogisticSP,khan12_multinomial,Carbonetto2012ScalableVI,wand2017fast,Durante_Rigon_2019,jin2025variational}, which admit the following closed-form expectations
\begin{eqnarray}
\label{min_bl_pg_el}
\begin{split}
	&\EX_{q}[\hbound_{\textsc{bl}}(\bx_i^\top \bbeta \mid \zeta_i)]  = h(\zeta_i) + h'(\zeta_i) (\bx_i^\top \bmu - \zeta_i) - 0.25 (\EX_{q}[(\bx_i^\top \bbeta)^2] - 2 \,\bx_i^\top \bmu \,\zeta_i + \zeta_i^2)/2,  \qquad \ \\
	&\EX_{q}[\hbound_{\textsc{pg}}(\bx_i^\top \bbeta \mid \zeta_i)]  = h(\zeta_i) - w_{\textsc{pg}}(\zeta_i) (\EX_{q}[(\bx_i^\top \bbeta)^2] - \zeta_i^2)/2,
\end{split}
\end{eqnarray}
where {$\EX_{q}[(\bx_i^\top \bbeta)^2] = (\bx_i^\top \bmu)^2 + \bx_i^\top \bOmega \,\bx_i$} under the selected Gaussian approximating family $\mathcal{Q}$. Crucially, this choice of $\mathcal{Q}$ yields a closed-form expression also for  $\EX_{q}[|\bx_i^\top \bbeta|] $,~namely~$\EX_{q}[|\bx_i^\top \bbeta|] = (\bx_i^\top \bmu)[ 2 \,\Phi(\bx_i^\top \bmu; \bx_i^\top \bOmega \bx_i) - 1 ] + 2 (\bx_i^\top \bOmega \bx_i) \phi(\bx_i^\top \bmu; \bx_i^\top \bOmega \bx_i)$. Hence, replacing the above quadratic minorizers with the newly-proposed \textsc{pq} one (see Section~\ref{sec_PLQ}) ensures that the variational expectation can still be derived in closed form.   In particular, we obtain
\begin{eqnarray}
\label{min_pq_el}
\begin{split}
	\EX_{q}[\hbound_{\textsc{pq}}(\bx_i^\top \bbeta \mid \zeta_i)] & = h(\zeta_i) - w_{\textsc{pq}}(\zeta_i) (\EX_{q}[(\bx_i^\top \bbeta)^2] - \zeta_i^2)/2 - \nu_{\textsc{pq}}(\zeta_i) (\EX_{q}[|\bx_i^\top \bbeta|] - |\zeta_i|), \quad 
\end{split}
\end{eqnarray}
where $ \nu_{\textsc{pq}}(\zeta_i)$ is defined in \eqref{plq_original_space}. Besides preserving analytical tractability, the above  \textsc{pq} bound achieves improved tightness in characterizing the original \textsc{elbo}, relative to the approximations~induced by the \textsc{bl} and \textsc{pg} minorizers. More specifically, as a direct consequence  of Lemma~\ref{Lemma_optimality_pg_1D_space} and Proposition~\ref{prop:PLQ_dom_PG}, in combination with \eqref{logit_elbo}, we have that
\begin{align*}
	& \textsc{elbo}[\phi(\bbeta - \bmu; \bOmega)]
	\geq \textsc{elbo}_{\textsc{pq}}[\phi(\bbeta - \bmu; \bOmega) \mid \bzeta] \geq \\
	& \qquad \geq \textsc{elbo}_{\textsc{pg}}[\phi(\bbeta - \bmu; \bOmega) \mid \bzeta] 
	\geq \textsc{elbo}_{\textsc{bl}}[\phi(\bbeta - \bmu; \bOmega) \mid \bzeta],
\end{align*}
which further implies
\begin{align*}
	& \textsc{elbo}[\phi(\bbeta - \bmu; \bOmega)]
	\geq {\max}_{\bzeta} \,\textsc{elbo}_{\textsc{pq}}[\phi(\bbeta - \bmu; \bOmega) \mid \bzeta] \geq \\
	& \qquad \geq {\max}_{\bzeta} \,\textsc{elbo}_{\textsc{pg}}[\phi(\bbeta - \bmu; \bOmega) \mid \bzeta]
	\geq {\max}_{\bzeta} \,\textsc{elbo}_{\textsc{bl}}[\phi(\bbeta - \bmu; \bOmega) \mid \bzeta],
\end{align*}
for any $\bmu$ and $\bOmega$. As discussed in Section~\ref{varb}, this improved tightness is expected to yield more accurate posterior approximations under the \textsc{pq} bound than those obtained by \textsc{bl} and~\textsc{pg} \citep[e.g.,][]{Ormerod_2010_VB, Blei_2017}.

In Sections \ref{subsec_message_passing}--\ref{subsec_zeta_update}, we detail the steps of a simple coordinate-ascent recursion for maximizing the \textsc{elbo} lower bounds induced by the \textsc{bl}, \textsc{pg} and \textsc{pq} minorizers. As discussed above, these three alternative objective functions are obtained by replacing  $\EX_{q}[h(\bx_i^\top \bbeta)]$ in \eqref{logit_elbo}, with $\EX_{q}[\hbound_{\textsc{bl}}(\bx_i^\top \bbeta \mid \zeta_i)] $, $\EX_{q}[\hbound_{\textsc{pg}}(\bx_i^\top \bbeta \mid \zeta_i)] $ and $\EX_{q}[\hbound_{\textsc{pq}}(\bx_i^\top \bbeta \mid \zeta_i)] $, respectively, whose closed-form expressions are available in  \eqref{min_bl_pg_el}--\eqref{min_pq_el}. This yields a tractable routine relying on the generic recursion
\begin{align}
	\label{vb_update_mu_omega}
	(\bmu^{(t+1)}, \bOmega^{(t+1)}) & = \argmax_{(\bmu, \bOmega)} \; \textsc{elbo}[\phi(\bbeta - \bmu; \bOmega) \mid \bzeta^{(t)}], \\
	\label{vb_update_zeta}
	\bzeta^{(t+1)} & = \argmax_{\bzeta} \; \textsc{elbo}[\phi(\bbeta - \bmu^{(t+1)}; \bOmega^{(t+1)}) \mid \bzeta],
\end{align}
where \eqref{vb_update_mu_omega} leverages variational message passing (\textsc{vmp}) \citep[e.g.,][]{Knowles_Minka_NIPS2011} (see Section \ref{subsec_message_passing}), while \eqref{vb_update_zeta} admits closed-form solution (see Section \ref{subsec_zeta_update}). 

Convergence of such a routine is assessed by monitoring the relative and absolute increments of $\textsc{elbo}[\phi(\bbeta - \bmu; \bOmega) \mid \bzeta]$ from $(t)$ to $(t+1)$, and stopping when both of these increments~are~below a $10^{-10}$ threshold.

\vspace{25pt}
\subsection{Updating the variational distribution via message passing}
\label{subsec_message_passing}
When lack of conditional conjugacy hinders closed form maximization in \eqref{vb_update_mu_omega}, popular~solutions rely on a natural gradient step, yielding to the so-called  variational message passing (\textsc{vmp}) \citep[][]{Knowles_Minka_NIPS2011, Wand_2014_KMW}. As shown in \cite{Castiglione_2025_pGLMM},~for~Gaussian~\textsc{vb} under generalized linear models, the \textsc{vmp} update proposed by \cite{Wand_2014_KMW} reduces~to~the iterative re-weighted least square recursion
\begin{equation}
	\label{vmp_update}
	\bOmega^{(t+1)} = \big( \bX^\top \bW_{\textsc{vmp}}^{(t)} \bX + \bOmega_0^{-1} \big)^{-1}, \quad
	\bmu^{(t+1)} = \bOmega^{(t+1)} \bX^\top \bW_{\textsc{vmp}}^{(t)} \,\bz_{\textsc{vmp}}^{(t)}.
\end{equation}
Here \smash{$\bz_{\textsc{vmp}}^{(t)} = ( z_{\textsc{vmp},1}^{(t)}, \dots, z_{\textsc{vmp},n}^{(t)} )^\top$} and \smash{$\bW_{\textsc{vmp}}^{(t)} = \diag(\{ \omega_{\textsc{vmp},i}^{(t)} \}_{i=1}^{n} )$} are, respectively, a pseudo-data vector and a weight matrix, whose $i$-th elements are defined as
\begin{align*}
	z_{\textsc{vmp},i}^{(t)} = \bx_i^\top \bmu^{(t)} - \frac{(y_i - 0.5) +  \E_{q^{(t)}}[\hbound'(\bx_i^\top \bbeta \mid \zeta_i^{(t)})]}{\E_{q^{(t)}}[\hbound''(\bx_i^\top \bbeta \mid \zeta_i^{(t)})]}, \quad
	\omega_{\textsc{vmp},i}^{(t)} = - \E_{q^{(t)}}[\hbound''(\bx_i^\top \bbeta \mid \zeta_i^{(t)})],
\end{align*}
where $\hbound'(r \mid \zeta)$ and $\hbound''(r \mid \zeta)$ are the first and second order derivatives of  $\hbound(r \mid \zeta)$.

Under the \textsc{bl} lower bound, it is easy to show that the \textsc{vmp} pseudo-data and weight are given by \smash{$z_{\textsc{vmp},i}^{(t)} = 4 (y_i - \pi(\zeta_i^{(t)}) + 0.25 \cdot \zeta_i^{(t)})$} and \smash{$\omega_{\textsc{vmp},i}^{(t)} = 0.25$}, thus leading to the fixed-point update
\begin{align*}
	\bOmega_{\textsc{bl}}^{(t+1)} = (0.25 \cdot \bX^\top \bX + \bOmega_0^{-1})^{-1}, \quad
	\bmu_{\textsc{bl}}^{(t+1)} = \bOmega_{\textsc{bl}}^{(t+1)} \bX^\top (\by - \bpi^{(t)} + 0.25 \cdot \bzeta^{(t)}),
\end{align*}
where \smash{$\bpi^{(t)} = (\pi(\zeta_1^{(t)}), \dots, \pi(\zeta_n^{(t)}))^\top$}. 
Note that the \textsc{bl} bound on the log-likelihood curvature~translates into the \textsc{vb} covariance matrix $\bOmega_{\textsc{bl}}^* = (0.25 \cdot \bX^\top \bX + \bOmega_0^{-1})^{-1}$, which is constant across all the iterations. However, such a simplification, while convenient from a computational viewpoint, is expected to worsen the quality of the resulting approximation. This is also consistent with~the~relative tightness of the \textsc{bl} bound~and~with~the~empirical results we discuss in Section~\ref{sec_application}.

As for the \textsc{pg} lower bound, the resulting \textsc{vmp} pseudo-data and the corresponding weights are given by \smash{$z_{\textsc{vmp},i}^{(t)} = (y_i - 0.5) / w_{\textsc{pg}}(\zeta_i^{(t)})$} and \smash{$\omega_{\textsc{vmp},i}^{(t)} = w_{\textsc{pg}}(\zeta_i^{(t)})$}, yielding to the \textsc{vmp} update
\begin{align*}
	\bOmega_{\textsc{pg}}^{(t+1)} = (\bX^\top \bW_{\textsc{pg}}^{(t)} \bX + \bOmega_0^{-1})^{-1}, \quad
	\bmu_{\textsc{pg}}^{(t+1)} = \bOmega_{\textsc{pg}}^{(t+1)} \bX^\top (\by - 0.5 \cdot {\bf 1}_n),
\end{align*}
where \smash{$\bW_{\textsc{pg}}^{(t)} = \diag(\{ w_{\textsc{pg}}(\zeta_i^{(t)}) \}_{i=1}^{n})$}.
Not surprisingly, \textsc{vmp} for \textsc{pg} is equivalent to the \textsc{vb} update proposed by \cite{Jaakkola_Jordan_2000} and further discussed by \cite{Durante_Rigon_2019}.~This happens because under conditionally conjugate models (as in the \textsc{pg} case) \textsc{vmp} reduces to \textsc{cavi}, as discussed in, e.g., \cite{Knowles_Minka_NIPS2011} and \cite{Tan_Nott_2013_GLMM}. 
Unlike the \textsc{bl} approximation, \textsc{pg} allows for a more flexible form of the covariance matrix, where each observation enters with a weight, thus iteratively  adapting the bound to the local curvature of the original~\textsc{elbo}.

Finally, under the newly-proposed \textsc{pq} lower bound the \textsc{vmp} pseudo-data and weight are
\vspace{-2pt}
\begin{align*}
	z_{\textsc{vmp},i}^{(t)} & = \bx_i^\top \bmu_{\textsc{pq}}^{(t)} + \frac{y_i - 0.5 - w_{\textsc{pq}}(\zeta_i^{(t)}) \bx_i^\top \bmu_{\textsc{pq}}^{(t)} - \nu_{\textsc{pq}}(\zeta_i^{(t)}) \EX_{q^{(t)}}[\sign(\bx_i^\top \bbeta)]}{w_{\textsc{pq}}(\zeta_i^{(t)}) + 2 \,\nu_{\textsc{pq}}(\zeta_i^{(t)})\phi(\bx_i^\top \bmu_{\textsc{pq}}^{(t)}; \bx_i^\top \bOmega_{\textsc{pq}}^{(t)} \bx_i)}, \\
	\omega_{\textsc{vmp},i}^{(t)} & = w_{\textsc{pq}}(\zeta_i^{(t)}) + 2  \,\nu_{\textsc{pq}}(\zeta_i^{(t)}) \phi(\bx_i^\top \bmu_{\textsc{pq}}^{(t)}{;} \bx_i^\top \bOmega_{\textsc{pq}}^{(t)} \bx_i),
\end{align*}
with \smash{$\EX_{q^{(t)}}[\sign(\bx_i^\top \bbeta)] = 2 \,\Phi(\bx_i^\top \bmu_{\textsc{pq}}^{(t)}{;} \bx_i^\top \bOmega_{\textsc{pq}}^{(t)} \bx_i) - 1$}.  Therefore, replacing these quantities in  \eqref{vmp_update} yields a closed-form  \textsc{vmp} update also for~\smash{$\bOmega_{\textsc{pq}}^{(t+1)}$} and \smash{$\bmu_{\textsc{pq}}^{(t+1)}$}. Differently from both \textsc{bl} and \textsc{pg}, the resulting \textsc{pq} recursion dynamically adapts both the weights and the pseudo-data of the iterative least squares cycle. This facilitates faster  convergence and guarantees improved tightness, thereby increasing the approximation quality of the resulting variational approximation. Such an intuition is also confirmed by the empirical experiments outlined in Section \ref{sec_application}, where the proposed \textsc{pq} minorizer achieves improved approximation accuracy than  \textsc{bl} and \textsc{pg} (see Figure~\ref{figure_tvd_fields}).

\vspace{15pt}
\subsection{Updating the local parameters via exact maximization}
\label{subsec_zeta_update}
Let us now focus on the update \eqref{vb_update_zeta} for the local parameters $\{ \zeta_i \}_{i=1}^n$ under the considered lower bounds (\textsc{bl}, \textsc{pg}, and \textsc{pq}). To this end, note that, for any $\bmu^*$ and $\bOmega^*$, $\argmax_{\bzeta} \,\textsc{elbo}[\phi(\bbeta - \bmu^*; \bOmega^*) \mid \bzeta] = \smash{\argmax_{\bzeta} \,\sum_{i=1}^{n} \EX_{q^*}[\hbound(\bx_i^\top \bbeta \mid \zeta_i)]}$, where $q^*(\bbeta) = \phi(\bbeta - \bmu^*; \bOmega^*)$.  Thanks to the separability of the objective function, we can optimize each $\zeta_i$, $i = 1, \dots, n$, separately. This implies finding the maximum of $Q(\zeta_i) := \EX_{q^*}[\hbound(\bx_i^\top \bbeta \mid \zeta_i)]$ for the three considered bounds, $\hbound_{\textsc{bl}}$, $\hbound_{\textsc{pg}}$, and $\hbound_{\textsc{pq}}$.

Under \textsc{bl}, the $\zeta$-update is obtained as the maximizer of
\begin{equation*}
	Q_{\textsc{bl}}(\zeta_i) = h(\zeta_i) + h'(\zeta_i) (\bx_i^\top \bmu^* - \zeta_i) - 0.25 (\EX_{q^*}[(\bx_i^\top \bbeta)^2] - 2 \,\bx_i^\top \bmu^* \,\zeta_i + \zeta_i^2)/2.
\end{equation*}
Taking the first derivative of $Q_{\textsc{bl}}(\zeta_i)$ and simplifying the redundant terms, we obtain $Q_{\textsc{bl}}'(\zeta_i) = [h''(\zeta_i) + 0.25] (\bx_i^\top \bmu^* - \zeta_i)$, which is equal to zero at $\zeta_i^* = \bx_i^\top \bmu^*$. 
Therefore, the \textsc{bl} update for $\bzeta$~is given by \smash{$\zeta_i^{(t+1)} = \bx_i^\top \bmu^{(t+1)}$, for $i = 1, \dots, n$}.

Similarly, under \textsc{pg}, the $\zeta$-update is the maximizer of $\EX_{q^*}[\hbound_{\textsc{pg}}(\bx_i^\top \bbeta \mid \zeta_i)]$.
Note that,~in~this~case $\EX_{q^*}[\hbound_{\textsc{pg}}(\bx_i^\top \bbeta \mid \zeta_i)]$ depends on $\zeta_i$ solely through its absolute value $|\zeta_i|$. Hence, we can restrict~our attention to the maximization over $\rho_i = |\zeta_i|$ of the objective function
\begin{equation*}
	Q_{\textsc{pg}}(\rho_i) = h(\rho_i) - w_{\textsc{pg}}(\rho_i) (\EX_{q^*}[(\bx_i^\top \bbeta)^2] - \rho_i^2)/2.
\end{equation*}
This has first derivative \smash{$Q_{\textsc{pg}}'(\rho_i) = h'(\rho_i) {+} w_{\textsc{pg}}(\rho_i) \rho_i - w_{\textsc{pg}}'(\rho_i) (\EX_{q^*}[(\bx_i^\top \bbeta)^2] - \rho_i^2)$}. Leveraging  the identity $h'(\rho_i) + w_{\textsc{pg}}(\rho_i) \rho_i = 0$, the first two terms cancel out and the solution to the equation $Q_{\textsc{pg}}'(\rho_i) = 0$ is obtained at $\rho_i^* = \EX_{q^*}[(\bx_i^\top \bbeta)^2]^{1/2}$. 
Thus, the \textsc{pg} update for $\bzeta$ is given by $\smash{\zeta_i^{(t+1)}} = \smash{\sign(\bx_i^\top \bmu^{(t+1)}) \EX_{q^{(t+1)}}[(\bx_i^\top \bbeta)^2]^{1/2}}$, for $i = 1, \dots, n$, where the term \smash{$\sign(\bx_i^\top \bmu^{(t+1)})$} is introduced to fix the sign ambiguity. The former is precisely the variational \textsc{em} update of each $\zeta_i$, $i=1, \ldots, n$ derived by \cite{Jaakkola_Jordan_2000}.

Finally, let us consider the $\zeta$-update for the proposed \textsc{pq} lower bound.
Similar to the \textsc{pg} case, also $\EX_{q^*}[\hbound_{\textsc{pq}}(\bx_i^\top \bbeta \mid \zeta_i)]$ depends on $\zeta_i$ only via its absolute value $|\zeta_i|$, allowing us to restrict our attention to the maximization over $\rho_i = |\zeta_i|$ of the objective function
\begin{equation*}
	Q_{\textsc{pq}}(\rho_i) := - \log\cosh(\rho_i/2) - w_{\textsc{pq}}(\rho_i) \big(\EX_{q^*}[(\bx_i^\top \bbeta)^2] - \rho_i^2 \big)/2 - \nu_{\textsc{pq}}(\rho_i) \big( \EX_{q^*}[|\bx_i^\top \bbeta|] - \rho_i \big) + c,
\end{equation*}
where $c$ is a constant term not depending on $\rho_i$.
In searching for a critical point of \smash{$Q_{\textsc{pq}}(\rho_i)$},~we~must solve again \smash{$Q_{\textsc{pq}}'(\rho_i) = 0$}, where 
\begin{align*}
	Q_{\textsc{pq}}'(\rho_i) = 
	& - 0.5 \cdot \tanh(\rho_i/2) + \nu_{\textsc{pq}}(\rho_i) + \rho_i w_{\textsc{pq}}(\rho_i) \\
	& - \nu_{\textsc{pq}}'(\rho_i) \big(\EX_{q^*}[|\bx_i^\top \bbeta|] - \rho_i \big) - w_{\textsc{pq}}'(\rho_i) \big(\EX_{q^*}[(\bx_i^\top \bbeta)^2] - \rho_i^2 \big)/2.
\end{align*}
From the explicit forms of $\nu_{\textsc{pq}}(\rho_i)$ and $w_{\textsc{pq}}(\rho_i)$, we note that $\nu_{\textsc{pq}}(\rho_i) + \rho_i w_{\textsc{pq}}(\rho_i) = \tanh(\rho_i/2)/2$, and that the first three terms of the above expression cancel out.
Moreover, by simple calculations, we obtain $\nu_{\textsc{pq}}'(\rho_i) = w_{\textsc{pq}}(\rho_i) - 0.25 {\cdot} \sech^2(\rho_i/2)$ and $w_{\textsc{pq}}'(\rho_i) = - (2/\rho_i) \nu_{\textsc{pq}}'(\rho_i)$.~This~implies
\begin{align*}
	Q'(\rho_i) 
	& = - \nu_{\textsc{pq}}'(\rho_i) \EX_{q^*}[|\bx_i^\top \bbeta|] - 0.5 \cdot w_{\textsc{pq}}'(\rho_i) \EX_{q^*}[(\bx_i^\top \bbeta)^2] + 0.5 \cdot \big( w_{\textsc{pq}}'(\rho_i) + (2/\rho_i) \nu_{\textsc{pq}}'(\rho_i) \big) \\
	& = 0.5 \cdot w_{\textsc{pq}}'(\rho_i) \big( \rho_i \cdot \EX_{q^*}[|\bx_i^\top \bbeta|] - \EX_{q^*}[(\bx_i^\top \bbeta)^2] \big),
\end{align*}
which is equal to zero for $\rho_i^* = \EX_{q^*}[(\bx_i^\top \bbeta)^2] / \EX_{q^*}[|\bx_i^\top \bbeta|]$. In turn, this gives the update
\begin{equation*}
	\zeta_i^{(t+1)} = \sign(\bx_i^\top \bmu^{(t+1)}) \frac{\EX_{q^{(t+1)}}[(\bx_i^\top \bbeta)^2]}{\EX_{q^{(t+1)}}[|\bx_i^\top \bbeta|]}, \quad i = 1, \dots, n,
\end{equation*}
where the term $\sign(\bx_i^\top \bmu^{(t+1)})$ is introduced to fix the sign ambiguity.

\vspace{16pt}
\subsection{Pseudo-code for variational inference under the \textsc{pq} minorizer}
Algorithm~1 provides the pseudo-code to perform variational inference in logistic regression under the proposed  \textsc{pq} minorizer, leveraging the results and derivations in Sections~\ref{subsec_message_passing}--\ref{subsec_zeta_update}.

\begin{figure}[t]
\centering
    \includegraphics[width=0.9\textwidth]{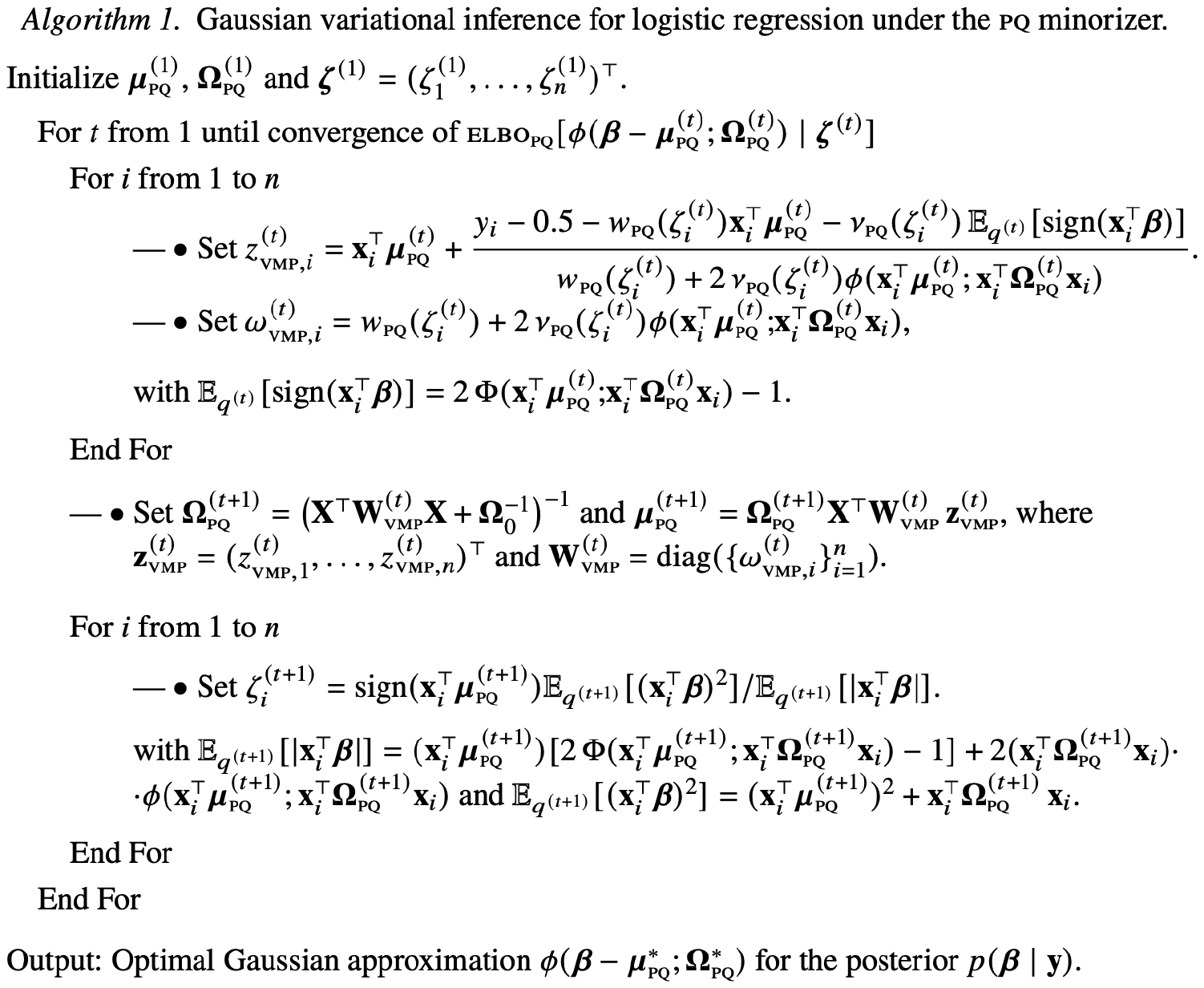}
    \label{algo}
    \vspace{15pt}
\end{figure}

\section{Motor-vehicle theft data from Portland}
In Section \ref{sec_application}, we analyze motor-vehicle theft data from Portland, Oregon, recorded in 2015 by the \textsc{usa} National Institute of Justice and publicly available at \url{https://nij.ojp.gov/funding/real-time-crime-forecasting-challenge-posting}. The original data comprise a total of 1,911 theft events along with the associated spatial locations in the city. To construct~the~dataset employed in our analysis, we overlay a regular square grid over the Portland map with~50 equally spaced segments along both longitude and latitude. Each cell within such a grid is then assigned a binary variable taking value $1$ if the observed number of thefts located in such a cell exceeds~the 75th percentile across all cells and  $0$ otherwise (cells with no available data are discarded from the analysis). As shown in the  left panel of Figure~\ref{figure_portland_data}, this yields a total of $n = 704$ pairs $(y_i, \bv_i)$, $i = 1, \dots, n$, where the binary response $y_i$ indicates whether the $i$th cell belongs or not to a high risk zone, while  $\bv_i \in \mathbb{R}^2 \subset \Gamma$ denotes its spatial location within the Portland map $\Gamma$.

Leveraging the above dataset, our goal is to model the spatial distribution of the high/low risk zone indicators over the Portland map. To this end, we employ a logistic regression for $y_i$ with suitably-specified linear predictor capturing variations in $\mbox{pr}(y_i=1 \mid \bv_i)$ across the city.  Consistent with this goal, we employ a basis expansion approach and set \smash{$\bx_i^\top \bbeta = \sum_{j=1}^{p} \psi_j(\bv_i)\beta_j$},~where $\psi_1(\bv), \dots, \psi_p(\bv)$ denote carefully specified local bases obtained via \textit{finite element methods} (\textsc{fem}) \citep[e.g.,][]{Lindgren_2011_SPDE, Sangalli_2013_PDE}. Under this finite element construction, each basis corresponds one-to-one with a mesh node (i.e., a triangle vertex) arising from a discretization of the spatial domain $\Gamma$ using a fine constrained Delaunay triangulation as implemented in the \texttt{R} package \texttt{fdaPDE} \citep{Sangalli_2021_PDE}; see the right panel Figure~\ref{figure_portland_data}. At the generic $j$th~mesh node, the associated basis $\psi_j(\bv)$ is a piecewise linear taking nonzero values only on the triangles adjacent to the $j$th node, and attaining its maximum value of 1 at such a node. In our implementation~the resulting number of \textsc{fem} bases is  $p = 3103$.

\begin{figure}[t!]
	\centering
	\includegraphics[width=0.94\textwidth]{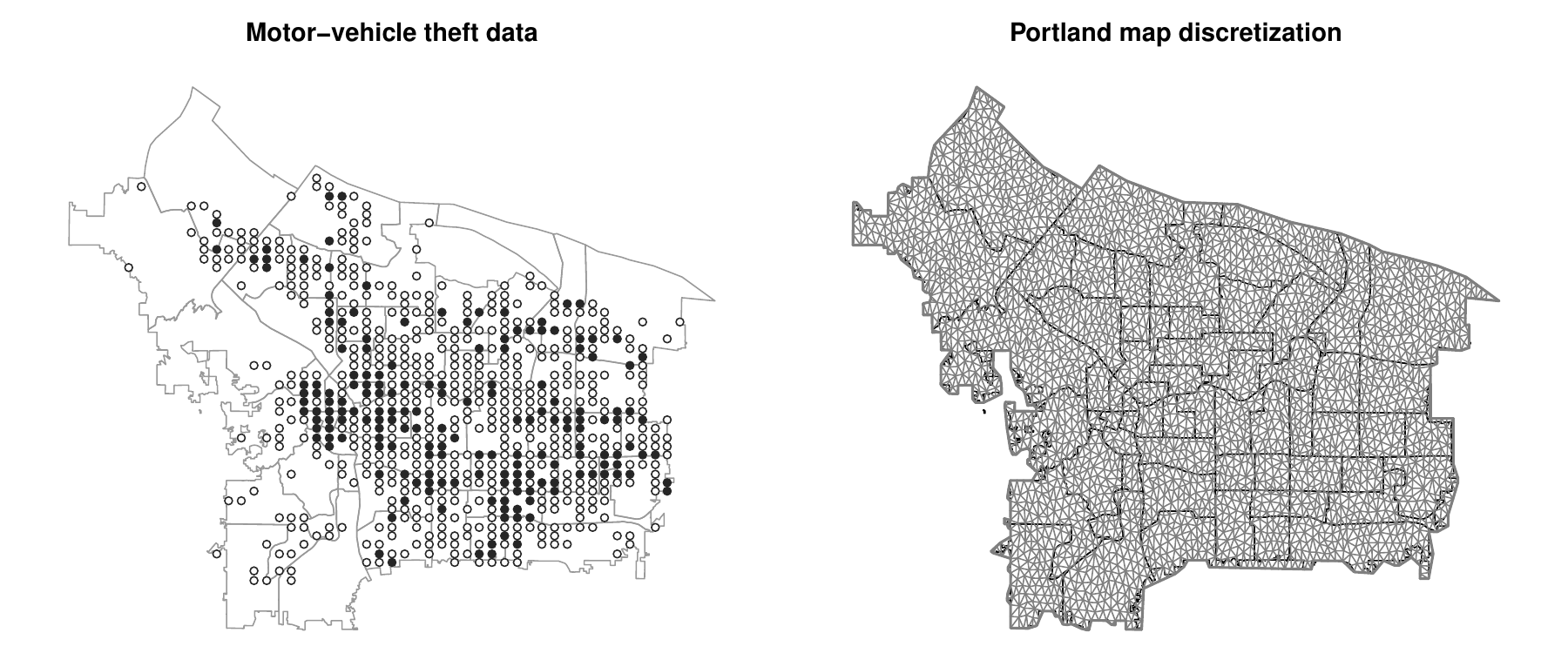}
	\vspace{-5pt}
	\caption{
		Left: Graphical representation of the dataset analyzed (white points indicate low risk cells, black points~denote high risk cells).
		Right: fine triangular discretization of the Portland map.
	}
	\label{figure_portland_data}
\end{figure}

As is common in nonparametric spatial regression relying on \textsc{fem} bases, we consider suitable regularizations among the basis coefficients, which can be introduced either explicitly through~penalized maximum likelihood estimation or implicitly via prior distributions under a Bayesian~approach (see Section \ref{sec_application}). In both settings, these regularizations are controlled by the matrix~$\bD$ that determines which linear combinations of $\bbeta$ are penalized.
Specifically, we adopt a Laplacian regularization  of the form \smash{$\bD = \bR_0^{-1/2}\bR_1$}, where $\bR_0$ and $\bR_1$~are~the~\textit{lumped-mass} and \textit{stiffness} matrices, respectively \citep[][]{Lindgren_2011_SPDE, Sangalli_2013_PDE}.
The former is a diagonal matrix with entries \smash{$\bR_{0,jj} = \sum_{k=1}^{p} \int_\Gamma \psi_j(\bv)\psi_k(\bv)\,d\bv$}, $j = 1, \dots, p$, while the latter is sparse with elements \smash{$\bR_{1,jk} = \int_\Gamma \langle \nabla \psi_j(\bv), \nabla \psi_k(\bv)\rangle\,d\bv$, $j,k = 1, \dots, p$}, where $\nabla \psi_j(\bv) = (\partial \psi_j(\bv)/\partial v_{i1}, \partial \psi_j(\bv)/\partial v_{i2})$ is a vector-valued piecewise-constant function that is nonzero only within the triangles adjacent~to~the $j$th node. Thanks to the local support of the basis functions, both the design and~penalty matrices are highly sparse, enabling efficient computation through sparse linear algebra routines.

\bibliography{ref.bib}

\begin{thebibliography}{45}
\providecommand{\natexlab}[1]{#1}
\providecommand{\url}[1]{\texttt{#1}}
\expandafter\ifx\csname urlstyle\endcsname\relax
  \providecommand{\doi}[1]{doi: #1}\else
  \providecommand{\doi}{doi: \begingroup \urlstyle{rm}\Url}\fi

\bibitem[Arnold and Tibshirani(2016)]{gen_lasso_dual_path_2016}
T.~B. Arnold and R.~J. Tibshirani.
\newblock Efficient implementations of the generalized lasso dual path
  algorithm.
\newblock \emph{J. Comput. Graph. Statist.}, 25\penalty0 (1):\penalty0 1--27,
  2016.

\bibitem[Bishop(2006)]{bishop2006pattern}
C.~M. Bishop.
\newblock \emph{Pattern Recognition and Machine Learning}.
\newblock Springer, 2006.

\bibitem[Bishop and Svens{\'e}n(2003)]{bishop2003bayesian}
C.~M. Bishop and M.~Svens{\'e}n.
\newblock Bayesian hierarchical mixtures of experts.
\newblock \emph{Proc. 19th Conf. Uncertain. Artif. Intell.}, 19:\penalty0
  57--64, 2003.

\bibitem[Blei et~al.(2017)Blei, Kucukelbir, and McAuliffe]{Blei_2017}
D.~M. Blei, A.~Kucukelbir, and J.~D. McAuliffe.
\newblock Variational inference: A review for statisticians.
\newblock \emph{J. Amer. Statist. Assoc.}, 112\penalty0 (518):\penalty0
  859–877, 2017.

\bibitem[B{\"o}hning(1992)]{Bohning1992MultinomialLR}
D.~B{\"o}hning.
\newblock Multinomial logistic regression algorithm.
\newblock \emph{Ann. Inst. Statist. Math.}, 44:\penalty0 197--200, 1992.

\bibitem[B{\"o}hning and Lindsay(1988)]{Bohning1988}
D.~B{\"o}hning and B.~Lindsay.
\newblock Monotonicity of quadratic-approximation algorithms.
\newblock \emph{Ann. Inst. Statist. Math.}, 40:\penalty0 641--663, 1988.

\bibitem[Boyd et~al.(2011)Boyd, Parikh, Chu, Peleato, Eckstein,
  et~al.]{Boyd_2011_ADMM}
S.~Boyd, N.~Parikh, E.~Chu, B.~Peleato, J.~Eckstein, et~al.
\newblock Distributed optimization and statistical learning via the alternating
  direction method of multipliers.
\newblock \emph{Found. Trends Mach. Learn.}, 3\penalty0 (1):\penalty0 1--122,
  2011.

\bibitem[Browne and McNicholas(2015)]{Multivariate_sharp_Qbounds}
R.~P. Browne and P.~D. McNicholas.
\newblock {Multivariate sharp quadratic bounds via $\mathbf{\Sigma}$-strong
  convexity and the Fenchel connection}.
\newblock \emph{Electron. J. Statist.}, 9\penalty0 (2):\penalty0 1913--1938,
  2015.

\bibitem[Carbonetto and Stephens(2012)]{Carbonetto2012ScalableVI}
P.~Carbonetto and M.~Stephens.
\newblock {Scalable variational inference for Bayesian variable selection in
  regression, and its accuracy in genetic association studies}.
\newblock \emph{Bayesian Anal.}, 7:\penalty0 73--108, 2012.

\bibitem[Castiglione and Bernardi(2025)]{Castiglione_2025_pGLMM}
C.~Castiglione and M.~Bernardi.
\newblock Non-conjugate variational bayes for pseudo-likelihood mixed effect
  models.
\newblock \emph{J. Comput. Graph. Statist.}, In press, 2025.

\bibitem[{De Leeuw} and Lange(2009)]{DeLeeuw_09_sharQ}
J.~{De Leeuw} and K.~Lange.
\newblock Sharp quadratic majorization in one dimension.
\newblock \emph{Comput. Statist. Data Anal.}, 53\penalty0 (7):\penalty0
  2471--2484, 2009.

\bibitem[Durante and Rigon(2019)]{Durante_Rigon_2019}
D.~Durante and T.~Rigon.
\newblock {Conditionally conjugate mean-field variational Bayes for logistic
  models}.
\newblock \emph{Statist. Sci.}, 34\penalty0 (3):\penalty0 472 -- 485, 2019.

\bibitem[Ermis and Bouchard(2014)]{Ermis_Bouchard_2014}
B.~Ermis and G.~Bouchard.
\newblock Iterative splits of quadratic bounds for scalable binary tensor
  factorization.
\newblock \emph{Proc. 30th Conf. Uncertain. Artif. Intell.}, 30:\penalty0
  192–199, 2014.

\bibitem[Friedman et~al.(2010)Friedman, Hastie, and
  Tibshirani]{friedman2010regularization}
J.~H. Friedman, T.~Hastie, and R.~Tibshirani.
\newblock Regularization paths for generalized linear models via coordinate
  descent.
\newblock \emph{J. Statist. Softw.}, 33:\penalty0 1--22, 2010.

\bibitem[Goldfarb and Idnani(1983)]{Goldfarb1983_QP}
D.~Goldfarb and A.~Idnani.
\newblock A numerically stable dual method for solving strictly convex
  quadratic programs.
\newblock \emph{Math. Program.}, 27\penalty0 (1):\penalty0 1--33, Sep 1983.

\bibitem[Hastie et~al.(2015)Hastie, Tibshirani, and
  Wainwright]{Hastie2015_book}
T.~Hastie, R.~Tibshirani, and M.~Wainwright.
\newblock \emph{Statistical Learning with Sparsity: The Lasso and
  Generalizations}.
\newblock Chapman \& Hall/CRC, 2015.
\newblock ISBN 1498712169.

\bibitem[Helwig(2025)]{helwig2025versatile}
N.~E. Helwig.
\newblock Versatile descent algorithms for group regularization and variable
  selection in generalized linear models.
\newblock \emph{J. Comput. Graph. Statist.}, 34\penalty0 (1):\penalty0
  239--252, 2025.

\bibitem[Hunter and Lange(2004)]{mm_Hunter2004}
D.~R. Hunter and K.~Lange.
\newblock A tutorial on {MM} algorithms.
\newblock \emph{Am. Statist.}, 58\penalty0 (1):\penalty0 30--37, 2004.

\bibitem[Jaakkola and Jordan(2000)]{Jaakkola_Jordan_2000}
T.~Jaakkola and M.~I. Jordan.
\newblock Bayesian parameter estimation via variational methods.
\newblock \emph{Statist. Comput.}, 10:\penalty0 25--37, 2000.

\bibitem[Javanmard et~al.(2025)Javanmard, Shao, and
  Bien]{javanmard2025prediction}
A.~Javanmard, S.~Shao, and J.~Bien.
\newblock Prediction sets for high-dimensional mixture of experts models.
\newblock \emph{J. R. Statist. Soc. Ser. B Statist. Methodol.}, 87:\penalty0
  850–871, 2025.

\bibitem[Jin et~al.(2025)Jin, Zhang, and Tang]{jin2025variational}
Y.~Jin, Y.~Zhang, and N.~Tang.
\newblock Variational bayesian logistic tensor regression with application to
  image recognition.
\newblock \emph{Bayesian Anal.}, In press, 2025.

\bibitem[Khan et~al.(2012)Khan, Mohamed, Marlin, and
  Murphy]{khan12_multinomial}
M.~Khan, S.~Mohamed, B.~Marlin, and K.~Murphy.
\newblock A stick-breaking likelihood for categorical data analysis with latent
  {G}aussian models.
\newblock \emph{Proc. 15th Int. Conf. Artif. Intell. Statist.}, 22:\penalty0
  610--618, 2012.

\bibitem[Knowles and Minka(2011)]{Knowles_Minka_NIPS2011}
D.~Knowles and T.~Minka.
\newblock Non-conjugate variational message passing for multinomial and binary
  regression.
\newblock \emph{Adv. Neural Inf. Process. Syst.}, 24:\penalty0 1--9, 2011.

\bibitem[Lee et~al.(2010)Lee, Huang, and Hu]{Lee2010SPARSELP}
S.~Lee, J.~Huang, and J.~Hu.
\newblock Sparse logistic principal components analysis for binary data.
\newblock \emph{Ann. Appl. Statist.}, 4:\penalty0 1579--1601, 2010.

\bibitem[Lin and Zhang(2006)]{lin2006component}
Y.~Lin and H.~H. Zhang.
\newblock Component selection and smoothing in multivariate nonparametric
  regression.
\newblock \emph{Ann. Statist.}, 34\penalty0 (5):\penalty0 2272--2297, 2006.

\bibitem[Lindgren et~al.(2011)Lindgren, Rue, and
  Lindstr{\"o}m]{Lindgren_2011_SPDE}
F.~Lindgren, H.~Rue, and J.~Lindstr{\"o}m.
\newblock An explicit link between {G}aussian fields and {G}aussian {M}arkov
  random fields: the stochastic partial differential equation approach.
\newblock \emph{J. R. Statist. Soc. Ser. B Statist. Methodol.}, 73\penalty0
  (4):\penalty0 423--498, 2011.

\bibitem[Marlin et~al.(2011)Marlin, Khan, and
  Murphy]{Marlin_ICML2011_Piecewise_Q_bounds}
B.~Marlin, M.~Khan, and K.~Murphy.
\newblock Piecewise bounds for estimating {B}ernoulli-logistic latent
  {G}aussian models.
\newblock \emph{Proc. 28th Int. Conf. Mach. Learn.}, 1:\penalty0 633--640,
  2011.

\bibitem[McLachlan and Krishnan(1996)]{em_McLachlan1996}
G.~McLachlan and T.~Krishnan.
\newblock \emph{The EM Algorithm and Extensions}.
\newblock Wiley, 1996.

\bibitem[Nesterov and Nemirovskii(1994)]{Nesterov_Nemirovskii_1994_book_QP}
Y.~Nesterov and A.~Nemirovskii.
\newblock \emph{Interior-Point Polynomial Algorithms in Convex Programming}.
\newblock Society for Industrial and Applied Mathematics, 1994.

\bibitem[Ormerod and Wand(2010)]{Ormerod_2010_VB}
J.~T. Ormerod and M.~P. Wand.
\newblock Explaining variational approximations.
\newblock \emph{Am. Statist.}, 64\penalty0 (2):\penalty0 140--153, 2010.

\bibitem[Pastukhov(2024)]{pastukhov2024fused}
V.~Pastukhov.
\newblock Fused lasso nearly-isotonic signal approximation in general
  dimensions.
\newblock \emph{Statist. Comput.}, 34\penalty0 (4):\penalty0 120, 2024.

\bibitem[Polson et~al.(2012)Polson, Scott, and Windle]{Polson2012}
N.~G. Polson, J.~G. Scott, and J.~Windle.
\newblock {Bayesian inference for logistic models using P{\'o}lya–Gamma
  latent variables}.
\newblock \emph{J. Amer. Statist. Assoc.}, 108:\penalty0 1339--1349, 2012.

\bibitem[Ren et~al.(2011)Ren, Du, Carin, and Dunson]{Ren2011LogisticSP}
L.~Ren, L.~Du, L.~Carin, and D.~Dunson.
\newblock Logistic stick-breaking process.
\newblock \emph{J. Mach. Learn. Res.}, 12:\penalty0 203--239, 2011.

\bibitem[Sangalli(2021)]{Sangalli_2021_PDE}
L.~M. Sangalli.
\newblock Spatial regression with partial differential equation regularisation.
\newblock \emph{Int. Statist. Rev.}, 89\penalty0 (3):\penalty0 505--531, 2021.

\bibitem[Sangalli et~al.(2013)Sangalli, Ramsay, and Ramsay]{Sangalli_2013_PDE}
L.~M. Sangalli, J.~O. Ramsay, and T.~O. Ramsay.
\newblock Spatial spline regression models.
\newblock \emph{J. R. Statist. Soc. Ser. B Statist. Methodol.}, 75\penalty0
  (4):\penalty0 681--703, 2013.

\bibitem[Tan and Nott(2013)]{Tan_Nott_2013_GLMM}
L.~S. Tan and D.~J. Nott.
\newblock Variational inference for generalized linear mixed models using
  partially noncentered parametrizations.
\newblock \emph{Statist. Sci.}, 28\penalty0 (2):\penalty0 168--188, 2013.

\bibitem[Tibshirani et~al.(2005)Tibshirani, Saunders, Rosset, Zhu, and
  Knight]{fused_lasso_2005}
R.~Tibshirani, M.~Saunders, S.~Rosset, J.~Zhu, and K.~Knight.
\newblock Sparsity and smoothness via the fused lasso.
\newblock \emph{J. R. Statist. Soc. Ser. B Statist. Methodol.}, 67\penalty0
  (1):\penalty0 91--108, 2005.

\bibitem[Tibshirani(2014)]{tibshirani2014adaptive}
R.~J. Tibshirani.
\newblock Adaptive piecewise polynomial estimation via trend filtering.
\newblock \emph{Ann. Statist.}, 42\penalty0 (1):\penalty0 285--323, 2014.

\bibitem[Tibshirani and Taylor(2011)]{gen_lasso_solution_path_2011}
R.~J. Tibshirani and J.~Taylor.
\newblock The solution path of the generalized lasso.
\newblock \emph{{Ann. Statist.}}, 39\penalty0 (3):\penalty0 1335--1371, 2011.

\bibitem[Wand(2014)]{Wand_2014_KMW}
M.~P. Wand.
\newblock Fully simplified multivariate normal updates in non-conjugate
  variational message passing.
\newblock \emph{J. Mach. Learn. Res.}, 15\penalty0 (39):\penalty0 1351--1369,
  2014.

\bibitem[Wand(2017)]{wand2017fast}
M.~P. Wand.
\newblock Fast approximate inference for arbitrarily large semiparametric
  regression models via message passing.
\newblock \emph{J. Amer. Statist. Assoc.}, 112\penalty0 (517):\penalty0
  137--168, 2017.

\bibitem[Wu and Lange(2010)]{mm_Wu2010}
T.~T. Wu and K.~Lange.
\newblock {The MM alternative to EM}.
\newblock \emph{Statist. Sci.}, 25\penalty0 (4):\penalty0 492--505, 2010.

\bibitem[Zhao et~al.(2012)Zhao, Ogden, and Reiss]{zhao2012wavelet}
Y.~Zhao, R.~T. Ogden, and P.~T. Reiss.
\newblock Wavelet-based lasso in functional linear regression.
\newblock \emph{J. Comput. Graph. Statist.}, 21\penalty0 (3):\penalty0
  600--617, 2012.

\bibitem[Zhu(2017)]{Zhu_2017_ADMM}
Y.~Zhu.
\newblock An augmented {ADMM} algorithm with application to the generalized
  lasso problem.
\newblock \emph{J. Comput. Graph. Statist.}, 26\penalty0 (1):\penalty0
  195--204, 2017.

\bibitem[Zou and Hastie(2005)]{elastic_net_paper}
H.~Zou and T.~Hastie.
\newblock Regularization and variable selection via the elastic net.
\newblock \emph{J. R. Statist. Soc. Ser. B Statist. Methodol.}, 67\penalty0
  (2):\penalty0 301--320, 2005.

\end{thebibliography}

\end{document}